\documentclass{article}
\IfFileExists{arxiv.sty}{\usepackage{arxiv}}{}
\usepackage[utf8]{inputenc}
\usepackage[T1]{fontenc}
\usepackage{amsmath,amssymb,amsthm,mathtools}
\usepackage{booktabs}
\usepackage{enumitem}
\usepackage{microtype}
\usepackage{xcolor}
\usepackage[colorlinks=true,linkcolor=blue!50!black,citecolor=blue!50!black,urlcolor=blue!50!black]{hyperref}
\usepackage[nameinlink,noabbrev]{cleveref}
\providecommand{\keywords}[1]{\par\medskip\noindent\textbf{Keywords: }\begingroup\def\and{; }#1\endgroup\par\medskip}

\newtheorem{theorem}{Theorem}[section]
\newtheorem{proposition}{Proposition}[section]
\newtheorem{lemma}{Lemma}[section]
\newtheorem{corollary}{Corollary}[section]
\newtheorem{assumption}{Assumption}[section]
\theoremstyle{definition}
\newtheorem{definition}{Definition}[section]
\theoremstyle{remark}
\newtheorem{remark}{Remark}[section]

\crefname{assumption}{Assumption}{Assumptions}
\Crefname{assumption}{Assumption}{Assumptions}
\crefname{theorem}{Theorem}{Theorems}
\Crefname{theorem}{Theorem}{Theorems}
\crefname{proposition}{Proposition}{Propositions}
\Crefname{proposition}{Proposition}{Propositions}
\crefname{lemma}{Lemma}{Lemmas}
\Crefname{lemma}{Lemma}{Lemmas}
\crefname{corollary}{Corollary}{Corollaries}
\Crefname{corollary}{Corollary}{Corollaries}
\crefname{definition}{Definition}{Definitions}
\Crefname{definition}{Definition}{Definitions}
\crefname{remark}{Remark}{Remarks}
\Crefname{remark}{Remark}{Remarks}

\newcommand{\R}{\mathbb R}
\newcommand{\E}{\mathbb E}

\newcommand{\F}{\mathcal F}

\newcommand{\Xset}{\mathcal X}
\newcommand{\Yset}{\mathcal Y}
\newcommand{\Pset}{\mathcal P}
\newcommand{\calN}{\mathcal N}
\newcommand{\eps}{\varepsilon}
\newcommand{\norm}[1]{\left\lVert #1\right\rVert}
\newcommand{\ip}[2]{\left\langle #1,#2\right\rangle}
\newcommand{\dist}{\operatorname{dist}}
\newcommand{\Fix}{\operatorname{Fix}}
\newcommand{\diam}{\operatorname{diam}}
\newcommand{\defeq}{\mathrel{:=}}

\title{Non-Expansive Two-Time-Scale Stochastic Approximation: A Fixed-Schedule One-Quarter Barrier and Bias-Corrected Acceleration}
\author{Dhruv Sarkar$^{1,2}$ \quad Vaneet Aggarwal$^3$\\
$^1$Indian Institute of Technology Kharagpur\\
$^2$Mohamed bin Zayed University of Artificial Intelligence\\
$^3$Purdue University\\
\texttt{dhruv.sarkar223@gmail.com}\quad \texttt{vaneet@purdue.edu}}
\date{}

\begin{document}
\maketitle

\begin{abstract}
Non-expansive two-time-scale stochastic approximation is governed by a slow stochastic Krasnoselskii--Mann fixed-point iteration rather than by contraction to a unique equilibrium. We study this regime under a contractive fast map and a non-expansive reduced slow map. We first prove a finite-horizon lower bound showing that, for any prescribed slow stepsize schedule $(\beta_k)$, the classical KM residual scale $(\sum_{i<N}\beta_i(1-\beta_i))^{-1}$ is worst-case sharp for the corresponding unregularized KM update. Combined with the raw fast-tracking leakage scale, this explains the previously observed $k^{-1/4+o(1)}$ last-iterate mean-square residual exponent.

We then introduce a residual-preconditioned slow oracle that cancels the first-order dependence on the fast tracking error. In a nested Tikhonov-KM algorithm, the uncorrected oracle yields total-sample rate $T^{-1/4+o(1)}$, while the corrected oracle yields $T^{-1/3+o(1)}$. This improvement comes from changing the slow-oracle bias from first order to second order in the fast error after all inner-loop samples are counted.

Finally, we show that the repeated inner-loop cost of the nested method can be avoided in a smooth derivative-oracle model. A single-loop algorithm that tracks both the fast equilibrium and the leakage preconditioner online achieves $T^{-1/2+o(1)}$ with $O(1)$ primitive samples per iteration. 
\end{abstract}


\section{Introduction}

Two-time-scale stochastic approximation (TTSA) studies coupled recursions of the form
\begin{equation}\label{eq:raw-ttsa-intro}
\begin{aligned}
        X_{k+1} &= X_k+\alpha_k\bigl(f(X_k,Y_k)-X_k+W^x_{k+1}\bigr),\\
        Y_{k+1} &= Y_k+\beta_k\bigl(g(X_k,Y_k)-Y_k+W^y_{k+1}\bigr),
\end{aligned}
\end{equation}
where $\beta_k/\alpha_k\to0$. The fast variable is meant to track the unique fixed point $x^*(Y_k)$ of $x\mapsto f(x,Y_k)$, while the slow variable follows the reduced map
\[
        h(y)\defeq g(x^*(y),y).
\]
When $h$ is contractive, the slow equilibrium is unique and the natural error metric is mean-square distance to that equilibrium. When $h$ is only non-expansive, the fixed-point set can be non-singleton and the natural finite-time metric becomes the fixed-point residual
\[
        p(y)\defeq h(y)-y.
\]
This regime appears in minimax optimization, variational inequalities, projected fixed-point schemes, and constrained stochastic approximation. It is also technically distinct from contractive TTSA: the slow recursion is a stochastic inexact Krasnoselskii--Mann (KM) iteration.

Prior finite-time work in this setting proves a last-iterate mean-square residual rate $k^{-1/4+o(1)}$ under martingale-difference observations, a contractive fast map, and a non-expansive reduced slow map~\cite{chandak_2026_nonexpansive_ttsa}. The purpose of this paper is to explain why that exponent appears and to identify what must be changed algorithmically in order to improve it.

The first point is that the KM part of the analysis is already sharp for a fixed slow stepsize schedule in mean square. If $Y_{k+1}=(1-\beta_k)Y_k+\beta_k h(Y_k)$ with $h$ non-expansive, then the exact deterministic mean-square residual scale is
\begin{equation}\label{eq:KM-scale-intro}
        \left(\sum_{i<N}\beta_i(1-\beta_i)\right)^{-1}.
\end{equation}
The matching upper estimates are classical in the Baillon--Bruck and Cominetti--Soto--Vaisman theory of KM asymptotic regularity~\cite{cominetti_soto_vaisman_2014}. We prove the corresponding finite-horizon lower bound by a planar rotation example. This result should be read as a fixed-schedule statement: for the same unregularized KM update and the same slow schedule, a uniform worst-case residual guarantee cannot improve on \eqref{eq:KM-scale-intro}. Beating the resulting exponent therefore requires changing the algorithmic regime, the allowed slow schedule, or the oracle structure; it cannot come from sharpening the worst-case KM dependence for the same schedule.

The second point is that the raw TTSA recursion prevents such an aggressive slow schedule. Under uniform Lipschitz continuity of the slow map in the fast coordinate, evaluating the slow map at $X_k$ rather than at $x^*(Y_k)$ creates the first-order perturbation
\[
        g(X_k,Y_k)-h(Y_k)
        =O\bigl(\norm{X_k-x^*(Y_k)}\bigr).
\]
The raw fast variable tracks a moving target $x^*(Y_k)$. The target moves by $O(\beta_k)$ per slow update, whereas the fast recursion contracts tracking error only at rate $O(\alpha_k)$ per step. Consequently the fast mean-square tracking scale contains the two terms
\[
        \alpha_k + \left(\frac{\beta_k}{\alpha_k}\right)^2,
\]
where $\alpha_k$ is the usual fast stochastic approximation variance floor and $(\beta_k/\alpha_k)^2$ is the deterministic lag behind the moving target. The standard raw-TTSA separation condition
\[
        \left(\frac{\beta_k}{\alpha_k}\right)^2\lesssim \alpha_k,
        \qquad\text{equivalently}\qquad
        \frac{\beta_k^2}{\alpha_k^3}\lesssim1,
\]
is precisely the condition that keeps the moving-target lag no larger than the intrinsic fast statistical accuracy. This condition should not be read as a universal information-theoretic restriction on every possible algorithm. Rather, it is the balance that preserves the canonical raw fast-tracking scale.

Allowing the moving-target lag to dominate does not, by itself, solve the raw problem. If $\alpha_k\asymp (k+1)^{-a}$ and $\beta_k\asymp (k+1)^{-b}$, the lag-dominated fast root-mean-square error has order $\beta_k/\alpha_k\asymp k^{-(b-a)}$. Because the raw slow oracle is only Lipschitz in the fast coordinate, this appears as a first-order inexact-KM perturbation. The weighted perturbation scale is then
\[
        \frac{\sum_{k<N}\beta_k(\beta_k/\alpha_k)}{\sum_{k<N}\beta_k(1-\beta_k)}
        \asymp N^{-\min\{1-b,\,b-a\}},
\]
up to logarithmic factors in the boundary case. Optimizing $\min\{1-b,b-a\}$ over $a>1/2$ and $a<b<1$ gives the supremum $1/4$, approached as $a\downarrow1/2$ and $b=(1+a)/2$. Thus the true raw obstruction is first-order fast-manifold leakage. The inequality $\beta_k^2/\alpha_k^3\lesssim1$ is one common way this obstruction appears, but violating it merely makes the moving-target lag itself the bottleneck. \Cref{subsec:raw-separation-origin} gives the complete derivation and sharpness examples.

We first analyze this correction in a nested Tikhonov-KM algorithm. At outer iteration $m$, the method computes an approximate fast equilibrium for the current slow point $Y_m$ and then performs one Tikhonov-regularized slow step. This model separates two effects that are coupled in the raw recursion: the outer non-expansive KM dynamics and the bias caused by imperfect fast equilibration. It also allows all inner-loop work to be counted explicitly in the total primitive-oracle complexity.

The key correction is local but applies directly in the nonlinear coordinates. Let
\[
        A(y)\defeq I-\nabla_x f(x^*(y),y),
        \qquad
        C(y)\defeq \nabla_x g(x^*(y),y),
        \qquad
        P_*(y)\defeq C(y)A(y)^{-1}.
\]
If $x=x^*(y)+e$, Taylor expansion gives
\[
        f(x,y)-x=-A(y)e+O(\norm e^2),
        \qquad
        g(x,y)=h(y)+C(y)e+O(\norm e^2).
\]
Therefore
\begin{equation}\label{eq:cancellation-intro}
        g(x,y)+P_*(y)\bigl(f(x,y)-x\bigr)-h(y)=O(\norm e^2).
\end{equation}
The preconditioned residual cancels the first-order leakage. Without \eqref{eq:cancellation-intro}, an inner mean-square fast error $\delta$ creates a slow bias of order $\delta^{1/2}$. With \eqref{eq:cancellation-intro}, it creates a slow bias of order $\delta$.

The same cancellation can be viewed as a first-order implementation of an implicit Newton point. For fixed $y$, the fast equilibrium is the zero of $r(x,y)=f(x,y)-x$, and one Newton step from $x$ is
\[
        \calN(x,y)=x+\bigl(I-\nabla_x f(x,y)\bigr)^{-1}\bigl(f(x,y)-x\bigr).
\]
Under a Lipschitz-Jacobian condition, $\calN(x,y)$ is second-order close to $x^*(y)$. Thus querying $g$ at $\calN(x,y)$ would make the slow error quadratic in the fast tracking error. The analyzed nested correction uses the residual-preconditioned form \eqref{eq:cancellation-intro}, which produces the same first-order cancellation without requiring that the algorithm literally form $\calN(x,y)$.

The outer method we use for the main sample-complexity theorem is a one-sample Tikhonov-regularized KM recursion. This is an anchored stochastic approximation update,
\[
        Y_{m+1}=\Pi_{\Yset}\left[Y_m+\beta\bigl(\widehat H_m-Y_m+\lambda(u-Y_m)\bigr)\right],
\]
where $\widehat H_m$ estimates either the uncorrected or the corrected reduced map. The Tikhonov term makes the effective outer map contractive with gap $\Theta(\lambda)$, while the stochastic noise remains multiplied by the small stepsize $\beta$. Thus the method remains compatible with constant-variance one-sample stochastic oracles.

\paragraph{Contributions.}
Our main contributions and total-sample conclusions are the following.
\begin{enumerate}[leftmargin=*]
\item \textbf{A sharp fixed-schedule KM barrier and diagnosis of the one-quarter exponent.}
We prove a finite-horizon lower bound showing that the classical non-expansive residual scale cannot be improved for the same unregularized slow KM update and the same slow stepsize schedule. Together with the raw fast-tracking scale $\alpha_k+(\beta_k/\alpha_k)^2$ and the first-order leakage calculation in \Cref{subsec:raw-separation-origin}, this explains the one-quarter exponent. The usual separation condition $\beta_k^2/\alpha_k^3\lesssim1$ is the subcase that keeps the moving-target lag at the fast statistical floor; if it is violated, the lag itself enters as a first-order inexact-KM perturbation and remains rate-limiting. This part is developed in \Cref{sec:km-lower}.
\item \textbf{A nested Tikhonov-KM bias correction.}
We introduce a residual-preconditioned slow oracle that cancels first-order fast-manifold leakage. In the nested total-sample model, the uncorrected Tikhonov-KM oracle yields
\[
        \E\norm{h(Y)-Y}^2\le T^{-1/4+o(1)},
\]
whereas the corrected oracle yields
\[
        \E\norm{h(Y)-Y}^2\le T^{-1/3+o(1)}.
\]
The improvement is a total-oracle-complexity certificate for the correction mechanism: the preconditioned residual changes the fast-manifold error seen by the slow recursion from first order to second order. This is a structured-oracle result: we assume differentiability in the fast variable, compact stabilization, moment bounds, and access to $P_*(y)$ or to an estimator satisfying the product-accuracy condition in \Cref{ass:precond}. It is not a claim that the raw recursion \eqref{eq:raw-ttsa-intro} is globally suboptimal under the original black-box assumptions. The correction and nested rates are proved in \Cref{sec:correction,sec:nested,sec:main-rates}.
\item \textbf{A single-loop learned-preconditioner route to the one-half exponent.}
The nested $T^{-1/3+o(1)}$ result leaves one algorithmic cost unresolved: a fresh inner fast solve is performed at each outer point. In a richer smooth derivative-oracle model, we avoid this repeated inner-loop cost by maintaining one running fast iterate $X_k$, one slow iterate $Y_k$, and a projected matrix iterate $P_k$ updated from stochastic samples of $I-\nabla_x f$ and $\nabla_x g$. Reusing these online trackers yields
\[
        \E\norm{h(Y)-Y}^2\le T^{-1/2+o(1)}
\]
with $O(1)$ primitive samples per iteration. This result is given in \Cref{sec:single-loop-half}.
\end{enumerate}

Throughout the paper, a statement of the form $T^{-c+o(1)}$ means that for every fixed $\eta>0$ the bound $C_\eta T^{-c+\eta}$ holds for all sufficiently large horizons, with constants depending on the fixed problem parameters and on $\eta$ but not on the horizon. We use the same convention for $N^{-c+o(1)}$.

\section{Related work and positioning}
\label{sec:related-work}

\paragraph{Stochastic approximation and two time scales.}
Stochastic approximation originates with the Robbins--Monro recursion
\cite{robbins_monro_1951} and has developed into a general framework for
noisy fixed-point, root-finding, control, and learning problems; see, for
example, the monographs \cite{benveniste_metivier_priouret_1990,
kushner_yin_2003,borkar_2008}.  Two-time-scale stochastic approximation
(TTSA) was introduced to analyze coupled recursions in which one component
tracks a fast equilibrium while the other evolves on a slower scale
\cite{borkar_1997}.  This mechanism underlies actor--critic reinforcement
learning \cite{konda_tsitsiklis_2004,ganesh2025sharper,ganesh2025order}, stochastic bilevel optimization
\cite{hong_wai_wang_yang_2023,kwon2023fully,gaur2026sample,wu2026convergence}, and a variety of stochastic control and
optimization algorithms.  Finite-time analyses of TTSA have largely focused
on settings in which the limiting fast and slow dynamics are contractive,
strongly monotone, or otherwise stable around a unique equilibrium; examples
include finite-sample bounds for linear TTSA \cite{dalal_2018,kaledin_2020}
and nonlinear TTSA \cite{doan_2021_nonlinear_ttsa}.  Recent asymptotic work
also studies more general two-time-scale behavior when one or both iterates
need not converge to a single point \cite{borkar_2024_general}.  The present
paper differs from the contractive finite-time literature by considering a
contractive fast map but only a non-expansive reduced slow map.  In this
regime the slow fixed-point set may be non-singleton, and the natural
finite-time quantity is the fixed-point residual rather than distance to a
distinguished equilibrium.

\paragraph{Non-expansive TTSA and the one-quarter residual rate.}
The closest predecessor is the finite-time analysis of non-expansive TTSA by
Chandak \cite{chandak_2026_nonexpansive_ttsa}.  That work observes that,
after eliminating the fast variable, the slow recursion can be viewed as a
stochastic inexact Krasnoselskii--Mann iteration, and proves a last-iterate
mean-square residual rate \(O(k^{-1/4+\epsilon})\) for arbitrarily small
\(\epsilon>0\), together with almost-sure convergence to the fixed-point set.
Our work takes this rate as a starting point rather than merely reproducing
it.  We decompose the \(1/4\) exponent into two separate effects: the sharp
fixed-schedule residual scale of the underlying non-expansive
Krasnoselskii--Mann dynamics, and the first-order leakage of fast tracking
error into the raw slow oracle.  This separation is important because it
shows that the bottleneck is not only the usual two-time-scale separation
condition; in lag-dominated schedules, the moving-target error itself enters
the inexact KM recursion at first order.  The proposed residual
preconditioning changes this first-order slow-oracle bias into a second-order
bias, which is what allows the nested corrected method to improve the total
primitive-sample rate from \(T^{-1/4+o(1)}\) to \(T^{-1/3+o(1)}\).

\paragraph{Krasnoselskii--Mann iteration and asymptotic regularity.}
The deterministic slow dynamics in our setting are Krasnoselskii--Mann (KM)
iterations, introduced by Mann \cite{mann_1953} and Krasnoselskii
\cite{krasnoselskii_1955} for fixed points of non-expansive maps.  A central
quantitative question for KM iteration is asymptotic regularity, namely the
decay of the residual \(\|T x_k-x_k\|\).  Baillon and Bruck
\cite{baillon_bruck_1996} identified the classical \(O(1/\sqrt{k})\)-type
residual behavior, and Cominetti, Soto, and Vaisman
\cite{cominetti_soto_vaisman_2014} established sharp variable-stepsize
estimates in terms of the Bernoulli-sum quantity
\(\sum_{i<k}\beta_i(1-\beta_i)\).  Our finite-horizon rotation construction
is a complementary lower-bound statement: for any fixed slow schedule, the
mean-square residual scale
\[
    \left(\sum_{i<N}\beta_i(1-\beta_i)\right)^{-1}
\]
cannot be uniformly improved for the same unregularized KM update.  This
result should be interpreted as sharpness of the fixed-schedule KM scale, not
as an impossibility theorem for all anchored, regularized, variance-reduced,
or otherwise modified fixed-point schemes.

\paragraph{Stochastic non-expansive fixed-point methods.}
There is also a growing literature on stochastic fixed-point algorithms for
non-expansive operators.  Bravo and Cominetti
\cite{bravo_cominetti_2024} study stochastic KM iterations for
finite-dimensional non-expansive maps and derive almost-sure convergence and
nonasymptotic residual bounds.  Bravo and Contreras
\cite{bravo_contreras_2026} analyze stochastic Halpern iteration with
minibatching and variance reduction, obtaining improved oracle-complexity
bounds for stochastic non-expansive fixed-point problems.  Cortild and Cartis
\cite{cortild_cartis_2026} recently weakened the variance assumptions for
stochastic KM iteration in Hilbert spaces.  These works are single-level
stochastic fixed-point methods: the noisy oracle is a direct stochastic
approximation of the non-expansive map.  In contrast, our slow oracle is
endogenously biased by a coupled fast recursion.  The main obstruction is
therefore not just martingale noise in the non-expansive map, but the
systematic displacement \(X_k-x^*(Y_k)\) of the fast iterate from the moving
fast equilibrium.  The residual-preconditioned oracle introduced here is
designed specifically to cancel this fast-manifold leakage.

\paragraph{Anchoring, Halpern iteration, and Tikhonov regularization.}
Anchoring and Tikhonov regularization are classical tools for recovering
strong convergence or better stability in non-expansive fixed-point problems.
Browder's approximants \cite{browder_1967}, Halpern iteration
\cite{halpern_1967}, and Wittmann's convergence theorem
\cite{wittmann_1992} are foundational examples.  Our outer recursion uses a
Tikhonov term in a related but finite-time stochastic role: the regularized
map has an effective contraction gap of order \(\lambda\), while its fixed
point remains an \(O(\lambda)\)-residual point for the original
non-expansive map.  This produces the finite-time decomposition
\[
    \lambda^2
    + \exp(-cN\beta\lambda)
    + \frac{\beta}{\lambda}
    + \frac{\epsilon_H^2}{\lambda^2},
\]
separating Tikhonov residual, finite-time contraction, stochastic variance,
and slow-oracle bias.  The improvement in the nested corrected algorithm is
therefore not due to anchoring alone: the uncorrected nested Tikhonov method
still gives the one-quarter total-sample exponent.  The improvement comes
from reducing the bias term \(\epsilon_H^2\) by canceling the first-order
fast error.

\paragraph{Connections to monotone operators, variational inequalities, and
minimax problems.}
Non-expansive fixed-point formulations are ubiquitous in monotone operator
theory, splitting methods, variational inequalities, and saddle-point
optimization; see \cite{bauschke_combettes_2017,facchinei_pang_2003} for
background.  Classical extragradient and mirror-prox methods for monotone
variational inequalities \cite{korpelevich_1976,nemirovski_2004} can also be
viewed through fixed-point or residual mappings.  These connections motivate
the use of a fixed-point residual as the performance criterion when the
solution set is non-singleton.  The present work is not an extragradient or
VI gap analysis.  Instead, it studies a two-time-scale stochastic
approximation whose reduced slow dynamics are abstractly non-expansive.  The
results therefore apply at the level of the reduced fixed-point map and
isolate a phenomenon---fast-manifold leakage into a non-expansive KM
recursion---that is hidden in single-level operator-splitting analyses.

\paragraph{Implicit sensitivity, bilevel optimization, and learned
preconditioners.}
The correction
\[
    P^*(y)
    =
    \nabla_x g(x^*(y),y)
    \bigl(I-\nabla_x f(x^*(y),y)\bigr)^{-1}
\]
is closely related in spirit to implicit sensitivity calculations in bilevel optimization, but its role is different. Classical implicit-differentiation and hypergradient methods estimate or approximate inverse-Jacobian and inverse-Hessian effects in order to account for the dependence of an inner solution on an outer variable
\cite{ji_yang_liang_2021,arbel_mairal_2022,hong_wai_wang_yang_2023}. A complementary first-order line avoids explicit hypergradient or second-order computations; examples include the fully first-order stochastic bilevel method of Kwon et al.~\cite{kwon2023fully} and the bilevel reinforcement-learning sample-complexity setting of Gaur et al.~\cite{gaur2026sample}. The goal here is different from both lines: we do not compute a gradient of an outer objective. Instead, the inverse fast residual Jacobian is used to cancel the leading dependence of the slow fixed-point oracle on the fast tracking error. The single-loop learned-preconditioner algorithm is also analogous to amortized implicit-sensitivity methods in that it maintains an online estimate of a sensitivity object rather than restarting an inner solve at every outer point. In our setting, this online tracking removes the repeated inner-loop cost of the nested corrected scheme and yields the structured-oracle \(T^{-1/2+o(1)}\) residual rate.

\paragraph{Summary of positioning.}
In short, the paper sits at the intersection of finite-time TTSA,
non-expansive KM fixed-point theory, and implicit-sensitivity correction.
Relative to non-expansive TTSA \cite{chandak_2026_nonexpansive_ttsa}, it
explains why the raw one-quarter exponent appears and identifies the
first-order fast-manifold leakage as the algorithmic obstruction.  Relative
to deterministic and stochastic KM theory, it gives a fixed-schedule
finite-horizon sharpness construction and then changes the oracle structure
rather than trying to sharpen the same KM estimate.  Relative to bilevel and
implicit-differentiation methods, it uses an inverse fast-residual Jacobian
not to form a hypergradient, but to construct a slow fixed-point oracle whose
bias is quadratic in the fast tracking error.  This is the mechanism behind
the nested \(T^{-1/3+o(1)}\) and single-loop \(T^{-1/2+o(1)}\) improvements.

\section{Setting and assumptions}\label{sec:setup}

We work in Euclidean spaces. The fast variable lies in $\R^{d_x}$ and the slow variable lies in $\R^{d_y}$. Throughout, $\norm{\cdot}$ denotes the Euclidean norm.

\subsection{Mean fields and reduced map}

\begin{assumption}[Contractive fast map]\label{ass:fast-contract}
For every $y$, the map $x\mapsto f(x,y)$ is a contraction with constant $\mu\in[0,1)$:
\[
        \norm{f(x_1,y)-f(x_2,y)}\le \mu\norm{x_1-x_2}.
\]
Consequently, for every $y$ there is a unique fixed point $x^*(y)$ satisfying $f(x^*(y),y)=x^*(y)$.
\end{assumption}

\begin{assumption}[Non-expansive reduced slow map]\label{ass:slow-nonexp}
The reduced map
\[
        h(y)=g(x^*(y),y)
\]
is non-expansive on a closed convex set $\Yset\subset\R^{d_y}$:
\[
        \norm{h(y_1)-h(y_2)}\le \norm{y_1-y_2},
        \qquad y_1,y_2\in\Yset.
\]
The fixed-point set $Y^*\defeq\Fix(h)=\{y\in\Yset:h(y)=y\}$ is nonempty.
\end{assumption}

\begin{assumption}[Compact stabilization]\label{ass:compact}
There are compact convex sets $\Xset\subset\R^{d_x}$ and $\Yset\subset\R^{d_y}$ such that $x^*(y)\in\Xset$ for all $y\in\Yset$, $h(\Yset)\subseteq\Yset$, and all algorithmic queries are projected onto $\Xset$ and $\Yset$ when needed. Let
\[
        D_y\defeq \diam(\Yset),\qquad D_x\defeq \diam(\Xset).
\]
The anchor $u$ used in the Tikhonov schemes belongs to $\Yset$.
\end{assumption}

The assumptions in this subsection are standard structural assumptions for separating fast stability from slow fixed-point geometry in finite-time TTSA and stochastic KM analyses. Contractive fast mean fields are the usual stability condition in finite-time TTSA, while non-expansive reduced slow maps are the standard setting for KM residual analysis~\cite{borkar_1997,dalal_2018,doan_2021_nonlinear_ttsa,chandak_2026_nonexpansive_ttsa,bravo_cominetti_2024}. Compact stabilization is also standard when one wants uniform constants; it can be implemented by explicit projection or by a separate stability/Lyapunov argument. In this paper we use projection in order to keep the finite-time analysis self-contained.

\subsection{Stochastic oracles}

The primitive fast and slow observations are denoted
\[
        \mathsf F(x,y;\xi),\qquad \mathsf G(x,y;\zeta).
\]
They satisfy
\[
        \E[\mathsf F(x,y;\xi)]=f(x,y),
        \qquad
        \E[\mathsf G(x,y;\zeta)]=g(x,y).
\]
All random variables are adapted to the relevant algorithmic filtration, and independent fresh samples are used when conditional unbiasedness is invoked.

\begin{assumption}[Conditional bounded second and fourth moments]\label{ass:moments}
There is a constant $\sigma<\infty$ such that the primitive oracles are conditionally unbiased and have conditional moments up to order four uniformly over adapted queries. More precisely, for every sub-$\sigma$-field $\mathcal H$ representing the algorithmic past and every $\mathcal H$-measurable query $(X,Y)\in\Xset\times\Yset$, fresh samples satisfy almost surely
\[
        \E[\mathsf F(X,Y;\xi)\mid\mathcal H]=f(X,Y),
        \qquad
        \E[\mathsf G(X,Y;\zeta)\mid\mathcal H]=g(X,Y),
\]
and, for $q=2,4$,
\[
\begin{aligned}
        \E\!
        \left[
        \norm{\mathsf F(X,Y;\xi)-f(X,Y)}^q\mid\mathcal H
        \right]
        &\le \sigma^q,\\
        \E\!
        \left[
        \norm{\mathsf G(X,Y;\zeta)-g(X,Y)}^q\mid\mathcal H
        \right]
        &\le \sigma^q.
\end{aligned}
\]
All conditional expectations in the paper are with respect to fresh samples independent of the past, conditional on the current adapted query.
\end{assumption}

\begin{assumption}[Fast-variable Lipschitz continuity of the slow mean]\label{ass:g-lip}
There is a constant $L_g<\infty$ such that
\[
        \norm{g(x,y)-g(x',y)}\le L_g\norm{x-x'}
        \qquad
        \text{for all }x,x'\in\Xset,\ y\in\Yset .
\]
This assumption is needed only for guarantees involving the uncorrected nested oracle. Without it, closeness of the inner fast iterate to $x^*(y)$ need not imply closeness of $g(x,y)$ to $h(y)$.
\end{assumption}

The conditional unbiasedness, bounded-moment, and fast-coordinate Lipschitz conditions in \Cref{ass:moments,ass:g-lip} are standard finite-time stochastic-approximation oracle hypotheses. The moment assumptions give uniform martingale-noise control for the finite-horizon recursions, and the Lipschitz condition in \Cref{ass:g-lip} is the minimal regularity needed to convert inner fast error into slow-oracle bias for the uncorrected nested method. Similar unbiasedness, moment, projection, and Lipschitz hypotheses appear throughout finite-time TTSA and stochastic non-expansive fixed-point analyses~\cite{dalal_2018,kaledin_2020,doan_2021_nonlinear_ttsa,chandak_2026_nonexpansive_ttsa,bravo_cominetti_2024,bravo_contreras_2026}.

\subsection{Smoothness for bias correction}

The bias-corrected algorithm requires more structure than the raw TTSA recursion.

\begin{assumption}[Fast-variable differentiability]\label{ass:smooth}
There is an open convex neighborhood $\mathcal U$ of $\Xset\times\Yset$ such that $f$ and $g$ are twice continuously differentiable in the fast variable $x$ on $\mathcal U$. The derivatives $\nabla_x f$ and $\nabla_x g$ are jointly continuous in $(x,y)$ on $\mathcal U$, and there are constants $L_1,L_2,L_f,L_g^{\nabla}<\infty$ such that, for all relevant $(x,y)$,
\[
        \norm{\nabla_x g(x,y)}\le L_1,
        \qquad
        \norm{\nabla_x f(x,y)}\le \mu,
\]
and, uniformly for all $x,x'$ with $(x,y),(x',y)\in\mathcal U$,
\[
        \norm{\nabla_x f(x,y)-\nabla_x f(x',y)}
        \le L_f\norm{x-x'},
        \qquad
        \norm{\nabla_x g(x,y)-\nabla_x g(x',y)}
        \le L_g^{\nabla}\norm{x-x'}.
\]
In particular, after increasing $L_2$ if necessary,
\[
        \norm{\nabla_x f(x,y)-\nabla_x f(x^*(y),y)}
        +
        \norm{\nabla_x g(x,y)-\nabla_x g(x^*(y),y)}
        \le L_2\norm{x-x^*(y)}.
\]
The bound $\norm{\nabla_x f}\le\mu<1$ is consistent with the contraction assumption and ensures that $I-\nabla_x f$ is uniformly invertible on the stabilized query region. The Lipschitz bound on $\nabla_x f$ is used in \Cref{lem:newton-point}, where the Taylor expansion is centered at the current query point rather than at $x^*(y)$. For the Newton-point consequence involving $g(\calN(x,y),y)$, we also assume that $\calN(x,y)$ belongs to $\mathcal U$ for every relevant stabilized query $(x,y)$ and that $g$ is Lipschitz in its fast coordinate on the set of such Newton images. This is a local-query assumption; alternatively one may project $\calN(x,y)$ onto a compact neighborhood on which the same Lipschitz bound holds.
\end{assumption}

Define
\begin{equation}\label{eq:A-C-P}
        A(y)\defeq I-\nabla_x f(x^*(y),y),
        \qquad
        C(y)\defeq \nabla_x g(x^*(y),y),
        \qquad
        P_*(y)\defeq C(y)A(y)^{-1}.
\end{equation}
Since $\norm{\nabla_x f}\le\mu<1$, $A(y)$ is invertible and $\norm{A(y)^{-1}}\le(1-\mu)^{-1}$.

\begin{assumption}[Preconditioner access]\label{ass:precond}
The bias-corrected nested algorithm has access to $P_*(y)$ for every queried $y\in\Yset$. The main nested theorem treats this exact preconditioner call as a separate structured oracle whose cost is not included in the primitive $\mathsf F,\mathsf G$ sample count; if desired, one may add one unit of structured-oracle cost per outer step without changing the exponent. The main theorem is stated for this exact preconditioner. If $P_*(Y_m)$ is replaced by an adapted estimator $\widehat P_m$, the same proof remains valid only under direct product and moment conditions. Namely, for the inner fast mean-square error scale $\delta_m$, one needs
\[
        \E\!
        \left[
        \norm{\bigl(\widehat P_m-P_*(Y_m)\bigr)\bigl(f(\bar X_m,Y_m)-\bar X_m\bigr)}^2
        \right]
        \le C_P\delta_m^2,
\]
for the mean residual product, and also
\[
        \E\!
        \left[
        \norm{\bigl(\widehat P_m-P_*(Y_m)\bigr)\bigl(\mathsf F(\bar X_m,Y_m;\zeta)-f(\bar X_m,Y_m)\bigr)}^2
        \right]
        \le C_P
\]
so that the stochastic part of the corrected oracle keeps a uniformly bounded conditional second moment. A bound only on $\E\norm{\widehat P_m-P_*(Y_m)}^2$ is not sufficient in general, because the products with the fast residual and the fast observation noise can be correlated with the preconditioner error.
\end{assumption}

\begin{assumption}[Online derivative and preconditioner oracle]
\label{ass:online-precond}
For the online one-half theorem, define the fast-coordinate derivative fields
\[
        A_x(x,y)\defeq I-\nabla_x f(x,y),
        \qquad
        C_x(x,y)\defeq \nabla_x g(x,y).
\]
The algorithm has access to stochastic derivative observations
\[
        \mathsf A(x,y;\omega^A),\qquad \mathsf C(x,y;\omega^C)
\]
which are conditionally unbiased and have conditional moments up to order four uniformly over adapted queries. There is a constant $\sigma_A<\infty$ such that the following bounds hold. That is, for every sub-$\sigma$-field $\mathcal H$ and every $\mathcal H$-measurable query $(X,Y)\in\Xset\times\Yset$, fresh derivative samples satisfy almost surely
\[
        \E[\mathsf A(X,Y;\omega^A)\mid\mathcal H]=A_x(X,Y),
        \qquad
        \E[\mathsf C(X,Y;\omega^C)\mid\mathcal H]=C_x(X,Y),
\]
and, for $q=2,4$,
\[
        \E\!\left[\norm{\mathsf A(X,Y;\omega^A)-A_x(X,Y)}_F^q\mid\mathcal H\right]
        +
        \E\!\left[\norm{\mathsf C(X,Y;\omega^C)-C_x(X,Y)}_F^q\mid\mathcal H\right]
        \le \sigma_A^q.
\]
There is a compact convex matrix set $\Pset$ such that $P_*(y)\in\Pset$ for all $y\in\Yset$ and
\[
        \sup_{P\in\Pset}\norm P \le P_{\max}<\infty.
\]
The map $y\mapsto P_*(y)$ is Lipschitz on $\Yset$:
\[
        \norm{P_*(y)-P_*(y')}\le L_P\norm{y-y'}.
\]
The projected preconditioner recursion always projects onto $\Pset$.
\end{assumption}

\begin{assumption}[Moving-target regularity for the online Tikhonov theorem]\label{ass:moving-target}
The fast equilibrium map is Lipschitz on the stabilized slow set:
\[
        \norm{x^*(y)-x^*(y')}\le L_*\norm{y-y'}
        \qquad y,y'\in\Yset .
\]
For the horizon-tuned online theorem in \Cref{sec:single-loop-half}, the primitive fast update and corrected slow update are bounded uniformly over the projected preconditioner set: there is $B<\infty$ such that, almost surely for all queried points,
\[
        \norm{\mathsf F(x,y;\xi)-x}\le B,
\]
and, for every $P\in\Pset$ and $0<\lambda\le1$,
\[
        \norm{\mathsf G(x,y;\zeta^g)+P(\mathsf F(x,y;\zeta^f)-x)-y+\lambda(u-y)}\le B.
\]
The same theorem remains valid if these almost-sure bounds are replaced by conditional fourth-moment bounds strong enough to imply the fourth-moment recursions in \Cref{lem:single-fast,lem:precond-track}.
\end{assumption}

\begin{remark}[Strength and scope of the oracle model]
\Cref{ass:smooth,ass:precond,ass:online-precond,ass:moving-target} are structured-oracle assumptions beyond the raw non-expansive TTSA recursion. \Cref{ass:smooth} assumes fast-coordinate Jacobians and local Lipschitz regularity so that the fast residual Jacobian \(A(y)=I-\nabla_x f(x^*(y),y)\) is uniformly invertible and the cancellation in \Cref{prop:bias-order} is meaningful. Assumptions of this type are standard in implicit-differentiation and stochastic bilevel analyses, where inverse-Jacobian or inverse-Hessian effects are computed, approximated, or amortized~\cite{ji_yang_liang_2021,hong_wai_wang_yang_2023,arbel_mairal_2022}. \Cref{ass:precond} is stronger still: it assumes direct access to \(P_*(y)=C(y)A(y)^{-1}\), or to an estimator satisfying the product condition stated there. This preconditioner should therefore be regarded as a separate sensitivity oracle, not as an ordinary primitive \(\mathsf F,\mathsf G\) sample.

The online theorem replaces exact access to \(P_*(y)\) by stochastic observations of \(A_x(x,y)=I-\nabla_x f(x,y)\) and \(C_x(x,y)=\nabla_x g(x,y)\), together with a projected stochastic-approximation recursion for \(P_k\). Such derivative observations are natural when the primitive maps are generated by differentiable computational graphs, smooth simulators, or model-based stochastic systems. In those settings, Jacobian or Jacobian-vector information may be obtained through automatic differentiation~\cite{baydin_pearlmutter_radul_siskind_2018}, pathwise or reparameterized stochastic derivatives~\cite{kingma_welling_2014}, or simulation-sensitivity estimators such as infinitesimal-perturbation, likelihood-ratio, and score-function estimators~\cite{lecuyer_1990_ipa_sf_lr}. In bilevel optimization, the same modeling distinction appears between methods that use explicit implicit-gradient information~\cite{ji_yang_liang_2021,hong_wai_wang_yang_2023}, methods that amortize sensitivity objects~\cite{arbel_mairal_2022}, and fully first-order or surrogate approaches that avoid explicit hypergradient computations~\cite{kwon2023fully,gaur2026sample}.

The moving-target regularity in \Cref{ass:moving-target} follows from standard parametric contraction arguments. For example, if \(f\) is Lipschitz in \(y\) with constant \(L_y\) and \(x\mapsto f(x,y)\) is uniformly \(\mu\)-contractive, then \(\norm{x^*(y)-x^*(y')}\le L_y(1-\mu)^{-1}\norm{y-y'}\). Similarly, \(P_*(y)=C(y)A(y)^{-1}\) is Lipschitz whenever \(A(y)\) and \(C(y)\) are Lipschitz and \(A(y)^{-1}\) is uniformly bounded. The compact projection hypotheses are the finite-time stabilization device also used in TTSA and stochastic non-expansive fixed-point analyses~\cite{dalal_2018,kaledin_2020,doan_2021_nonlinear_ttsa,chandak_2026_nonexpansive_ttsa,bravo_cominetti_2024,bravo_contreras_2026}.

Thus the \(T^{-1/3+o(1)}\) and \(T^{-1/2+o(1)}\) results should be interpreted as structured-oracle results. They show what becomes possible when fast-manifold sensitivity information is available exactly or can be learned online; they are not claims about the original black-box TTSA model.
\end{remark}

\section{A sharp fixed-schedule KM barrier}\label{sec:km-lower}

This section proves that the KM mean-square residual scale is sharp for a fixed slow stepsize schedule and then computes how the one-quarter exponent follows from the standard two-time-scale separation.

\subsection{Finite-horizon lower bound}

The next theorem isolates the KM obstruction independently of fast tracking. Once a slow stepsize schedule $(\beta_k)$ has been fixed, it asks whether the worst-case non-expansive residual can decay faster than the classical KM scale. The following finite-horizon construction shows that a deterministic planar rotation already attains the inverse-$B_N$ scale.

\begin{theorem}[Finite-horizon lower bound for exact KM]\label{thm:KM-lower}
Let $0<\beta_k<1$ and define
\[
        B_N\defeq\sum_{k=0}^{N-1}\beta_k(1-\beta_k).
\]
For every $N$ such that $B_N\ge1/8$, there exists a deterministic, noiseless TTSA instance satisfying \Cref{ass:fast-contract,ass:slow-nonexp,ass:compact} and an initial point with $\dist(Y_0,Y^*)=1$ such that the exact slow recursion
\[
        Y_{k+1}=(1-\beta_k)Y_k+\beta_k h(Y_k)
\]
satisfies
\[
        \norm{h(Y_N)-Y_N}^2\ge \frac{1}{4B_N}.
\]
\end{theorem}

The proof is given in Appendix~\ref{app:km-proof}.

\begin{corollary}[Polynomial schedules]\label{cor:KM-poly}
If $\beta_k\asymp (k+1)^{-b}$ with $b\in(0,1)$, then for all sufficiently large $N$ there is a deterministic noiseless instance satisfying the assumptions of \Cref{thm:KM-lower} for which
\[
        \norm{h(Y_N)-Y_N}^2\ge cN^{-(1-b)}.
\]
\end{corollary}

The proof is included in Appendix~\ref{app:km-proof}.

The lower bound does not itself force $b>3/4$. It says that, for a chosen slow schedule $\beta_k\asymp(k+1)^{-b}$, the exact KM residual cannot be uniformly better than $N^{-(1-b)}$. The one-quarter exponent appears only after combining this sharp KM scale with the raw fast-tracking and first-order leakage scales derived next. The usual separation condition is one way these scales restrict the uncorrected recursion; the lag-dominated regime leads to the same exponent by a different bottleneck.

\subsection{Raw tracking, separation, and the lag-dominated regime}\label{subsec:raw-separation-origin}

Suppose the raw single-loop recursion uses polynomial stepsizes
\[
        \alpha_k\asymp (k+1)^{-a},\qquad \beta_k\asymp (k+1)^{-b}.
\]
The condition
\begin{equation}\label{eq:separation}
        \frac{\beta_k^2}{\alpha_k^3}\lesssim 1
\end{equation}
is often stated as a technical two-time-scale assumption. In the raw recursion it has a simple source: it is exactly the condition that the deterministic lag of the fast iterate behind the moving equilibrium $x^*(Y_k)$ is no larger than the usual fast stochastic approximation variance floor. This subsection also records what happens if one does not impose \eqref{eq:separation}: the lag is then larger than the fast statistical floor, but it enters the uncorrected slow KM recursion as a first-order perturbation and still does not permit a better than one-quarter raw exponent.

We spell this out at the level of the standard tracking recursion. Write
\[
        e_k\defeq X_k-x^*(Y_k),\qquad x^*_k\defeq x^*(Y_k).
\]
Ignore projection for notation; projection only decreases the distance to the current target because $x^*(Y_k)\in\Xset$. If the slow point were frozen at $Y_k$, contractivity of $x\mapsto f(x,Y_k)$ and conditional unbiasedness of the fast oracle would give, for sufficiently small $\alpha_k$,
\begin{equation}\label{eq:frozen-fast-recursion-main}
        \E\bigl[\norm{X_{k+1}-x^*_k}^2\mid\F_k\bigr]
        \le (1-c\alpha_k)\norm{e_k}^2+C\alpha_k^2 .
\end{equation}
Here the additive term $C\alpha_k^2$ is the one-step contribution of the fast martingale noise. In the actual TTSA recursion the target changes from $x^*(Y_k)$ to $x^*(Y_{k+1})$. Under the usual moving-target regularity used in raw TTSA finite-time analyses, $x^*$ is Lipschitz and the slow increment has conditional mean-square size $O(\beta_k^2)$. Therefore
\[
        \norm{x^*(Y_{k+1})-x^*(Y_k)}=O(\norm{Y_{k+1}-Y_k})=O(\beta_k)
\]
in mean square. Combining this drift with \eqref{eq:frozen-fast-recursion-main} and using Young's inequality with parameter proportional to $\alpha_k$ yields
\begin{equation}\label{eq:raw-fast-tracking-scale-recursion}
        u_{k+1}\le (1-c\alpha_k)u_k+C\alpha_k^2+C\frac{\beta_k^2}{\alpha_k},
        \qquad
        u_k\defeq \E\norm{X_k-x^*(Y_k)}^2 .
\end{equation}
The last term is the price of absorbing the cross term between the contracted fast error and the moving-target displacement. Dividing the additive terms in \eqref{eq:raw-fast-tracking-scale-recursion} by the contraction gap $\alpha_k$ gives the local tracking scale
\begin{equation}\label{eq:raw-fast-tracking-scale}
        u_k\lesssim \alpha_k+\left(\frac{\beta_k}{\alpha_k}\right)^2 .
\end{equation}
The first term is the intrinsic fast stochastic approximation accuracy. The second term is deterministic moving-target lag: even without fast noise, a filter that contracts at speed $\alpha_k$ cannot follow a target moving at speed $\beta_k$ more accurately than order $\beta_k/\alpha_k$ in norm.

The raw slow oracle sees the fast error to first order. Indeed, by fast-coordinate Lipschitz continuity,
\begin{equation}\label{eq:raw-leakage-main}
        \norm{g(X_k,Y_k)-h(Y_k)}^2
        \le L_g^2\norm{X_k-x^*(Y_k)}^2 .
\end{equation}
Thus the generic squared slow-oracle perturbation inherited by the raw recursion is of order
\[
        \alpha_k+\left(\frac{\beta_k}{\alpha_k}\right)^2 .
\]
If one wants the moving-target part not to dominate the canonical fast statistical scale $\alpha_k$, one must impose
\[
        \left(\frac{\beta_k}{\alpha_k}\right)^2\lesssim \alpha_k,
\]
which is exactly \eqref{eq:separation}. The next proposition records that neither the moving-target term in \eqref{eq:raw-fast-tracking-scale} nor the first-order transfer in \eqref{eq:raw-leakage-main} is a proof artifact.

\begin{proposition}[Origin and sharpness of the raw tracking and leakage scales]\label{prop:raw-separation-source}
In the raw, uncorrected TTSA model with only Lipschitz control of $g$ in the fast coordinate, the scale
\[
        \E\norm{X_k-x^*(Y_k)}^2
        \asymp
        \alpha_k+\left(\frac{\beta_k}{\alpha_k}\right)^2
\]
is worst-case sharp at the level of local tracking. The $\alpha_k$ term is attained by a scalar linear stochastic approximation with a fixed target and constant-variance noise. The $(\beta_k/\alpha_k)^2$ term is attained, even without noise, by a linear contractive fast map tracking the target generated by an exact non-expansive KM rotation. Moreover, under the Lipschitz-only slow-oracle assumption, the raw bias $g(X_k,Y_k)-h(Y_k)$ can be first order in $X_k-x^*(Y_k)$. Consequently, a worst-case raw TTSA analysis that insists that the moving-target lag remain at the intrinsic fast statistical scale must impose \eqref{eq:separation}.
\end{proposition}

The proof is given in Appendix~\ref{app:raw-separation-proof}. The proposition should be interpreted carefully. It does not say that every possible algorithm or every richer oracle model must satisfy \eqref{eq:separation}. It says that, for the raw recursion and under only a Lipschitz first-order slow oracle, the moving-target contribution and the first-order leakage are both genuinely present. Bias correction overcomes the resulting restriction precisely because it changes the slow oracle so that the first-order transfer in \eqref{eq:raw-leakage-main} is cancelled.

The remaining question is whether one can simply allow the moving-target term to dominate the fast statistical floor. The following scaling calculation shows that this does not improve the raw Lipschitz-only rate. When the lag dominates, the root-mean-square fast error is of order $\beta_k/\alpha_k$, and the raw slow oracle transfers this error to the KM recursion at first order. The following proposition gives the resulting exponent calculation.

\begin{proposition}[Lag-dominated raw schedules do not improve the one-quarter exponent]\label{prop:raw-lag-dominates-no-win}
Let
\[
        \alpha_k\asymp (k+1)^{-a},\qquad \beta_k\asymp (k+1)^{-b},
        \qquad a>\frac12,\qquad a<b<1.
\]
Suppose the raw slow recursion is affected by the first-order moving-target perturbation
\[
        \eta_k\asymp \frac{\beta_k}{\alpha_k}\asymp (k+1)^{-(b-a)} .
\]
In the standard inexact-KM residual algebra, such a perturbation contributes the weighted scale
\[
        E_N^{\rm lag}
        \defeq
        \frac{\sum_{k<N}\beta_k\eta_k}{\sum_{k<N}\beta_k(1-\beta_k)}
        \asymp
        N^{-\min\{1-b,\,b-a\}},
\]
up to a logarithmic factor when $2b-a=1$. Consequently, the best exponent allowed by the simultaneous exact-KM term $N^{-(1-b)}$ and the first-order lag perturbation is
\[
        \sup_{a>1/2}\sup_{a<b<1}\min\{1-b,\,b-a\}=\frac14,
\]
with the supremum approached as $a\downarrow1/2$ and $b=(1+a)/2$. Thus violating \eqref{eq:separation} does not improve the raw one-quarter exponent; it merely changes the active bottleneck from the fast statistical floor to the first-order moving-target lag.
\end{proposition}

The proof is included in Appendix~\ref{app:raw-separation-proof}. We now recover the usual separated calculation as a special case. Substituting the polynomial stepsizes into \eqref{eq:separation} gives
\[
        \frac{\beta_k^2}{\alpha_k^3}\asymp (k+1)^{3a-2b}.
\]
Thus \eqref{eq:separation} requires $3a-2b\le0$, or
\[
        b\ge \frac{3a}{2}.
\]
Since the usual Robbins--Monro square-summability condition $\sum_k\alpha_k^2<\infty$ requires $a>1/2$, the slow exponent must satisfy $b>3/4$. By \Cref{cor:KM-poly}, the sharp fixed-schedule KM mean-square residual scale is $N^{-(1-b)}$, hence the exponent cannot exceed $1/4$ in the separated regime. By \Cref{prop:raw-lag-dominates-no-win}, dropping the separation condition does not improve the raw scaling, because the lag-induced first-order perturbation then becomes the bottleneck. This is the sense in which the one-quarter exponent is a consequence of first-order fast-tracking leakage combined with the sharp KM mean-square scale, rather than a consequence of the displayed inequality \eqref{eq:separation} alone.

\section{Bias correction in the differentiable TTSA model}\label{sec:correction}

This section first presents the correction through an implicit Newton point. This viewpoint is not needed to run the nested algorithm, but it explains why the preconditioned residual removes the leading fast-manifold error.

\subsection{The implicit Newton point}\label{subsec:newton-view}

For fixed $y$, define the fast residual
\[
        r(x,y)\defeq f(x,y)-x.
\]
The fast equilibrium $x^*(y)$ is the unique zero of $r(\cdot,y)$. Newton's method for solving $r(x,y)=0$ gives
\begin{equation}\label{eq:newton-point}
        \calN(x,y)
        \defeq
        x+\bigl(I-\nabla_x f(x,y)\bigr)^{-1}\bigl(f(x,y)-x\bigr).
\end{equation}
The sign in \eqref{eq:newton-point} follows from $\nabla_x r(x,y)=\nabla_x f(x,y)-I$.

\begin{lemma}[Second-order accuracy of the implicit point]\label{lem:newton-point}
Suppose \Cref{ass:fast-contract,ass:compact,ass:smooth} hold. Then, for all relevant $(x,y)$,
\[
        \norm{\calN(x,y)-x^*(y)}
        \le
        \frac{L_f}{2(1-\mu)}\norm{x-x^*(y)}^2 .
\]
Consequently, if for the queried pair $(x,y)$ the point $\calN(x,y)$ lies in a set on which $g(\cdot,y)$ is $L_N$-Lipschitz uniformly in $y$, and if the segment joining $\calN(x,y)$ to $x^*(y)$ remains in that set, then
\[
        \norm{g(\calN(x,y),y)-h(y)}
        \le
        \frac{L_NL_f}{2(1-\mu)}\norm{x-x^*(y)}^2 .
\]
In particular, the displayed consequence holds if the Newton image of the stabilized query region is contained in an open fast-coordinate neighborhood on which $g$ is uniformly Lipschitz.
\end{lemma}

The proof is given in Appendix~\ref{app:bias-proofs}.

\paragraph{From Newton to residual preconditioning.}
The point $\calN(x,y)$ is the most transparent correction: it approximates the fast equilibrium to second order and would make the slow query $g(\calN(x,y),y)$ second-order accurate. The algorithms below use an equivalent first-order cancellation that avoids explicitly forming the Newton point. With $A(y)$, $C(y)$, and $P_*(y)$ defined in \eqref{eq:A-C-P}, Taylor expansion around $x^*(y)$ gives
\[
        f(x^*(y)+e,y)-(x^*(y)+e)=-A(y)e+O(\norm e^2),
\]
and
\[
        g(x^*(y)+e,y)=h(y)+C(y)e+O(\norm e^2).
\]
Therefore
\[
        g(x^*(y)+e,y)+P_*(y)\bigl(f(x^*(y)+e,y)-(x^*(y)+e)\bigr)-h(y)=O(\norm e^2),
\]
because $P_*(y)=C(y)A(y)^{-1}$. This is the residual-preconditioned form analyzed in the rest of the paper.

\subsection{Residual-preconditioned correction and oracle bias}\label{subsec:bias-order}

This subsection states the deterministic cancellation identity behind the corrected oracle. The proof is deferred to the appendix so that the main text emphasizes the algorithmic consequence: first-order fast-manifold leakage becomes second-order after residual preconditioning.

\begin{definition}[Uncorrected and corrected reduced queries]
For $x\in\Xset$ and $y\in\Yset$, define
\[
        H^{\rm raw}(x,y)\defeq g(x,y),
\]
and
\begin{equation}\label{eq:Hcorr}
        H^{\rm corr}(x,y)
        \defeq
        g(x,y)+P_*(y)\bigl(f(x,y)-x\bigr),
\end{equation}
where $P_*(y)$ is defined in \eqref{eq:A-C-P}.
\end{definition}

\begin{proposition}[First-order versus second-order slow bias]\label{prop:bias-order}
If \Cref{ass:g-lip} holds, then there is a constant $L_c<\infty$ such that for all $x\in\Xset$ and $y\in\Yset$,
\[
        \norm{H^{\rm raw}(x,y)-h(y)}
        \le L_c\norm{x-x^*(y)}.
\]
If \Cref{ass:fast-contract,ass:smooth} hold, then possibly increasing $L_c$ gives
\[
        \norm{H^{\rm corr}(x,y)-h(y)}
        \le L_c\norm{x-x^*(y)}^2.
\]
\end{proposition}

The proof is given in Appendix~\ref{app:bias-proofs}.

\section{Nested algorithms}\label{sec:nested}

The algorithms in this section are horizon-tuned. The horizon is denoted by $N$. An anytime version can be obtained by the standard doubling trick, which changes the displayed rates only by logarithmic factors; see Appendix~\ref{app:anytime}.

\subsection{The inner fast solver}

At outer step $m$, the slow variable $Y_m$ is held fixed. Starting from an arbitrary point $X_{m,0}\in\Xset$, run
\begin{equation}\label{eq:inner-loop}
        X_{m,t+1}
        =\Pi_{\Xset}\left[
        X_{m,t}+\eta_t\bigl(\mathsf F(X_{m,t},Y_m;\xi_{m,t+1})-X_{m,t}\bigr)
        \right],
        \qquad t=0,\ldots,n-1.
\end{equation}
Let $\bar X_m\defeq X_{m,n}$. The number $n$ will be chosen as a function of the outer horizon $N$.

\begin{proposition}[Uniform inner-loop accuracy]\label{prop:inner}
Suppose \Cref{ass:fast-contract,ass:compact,ass:moments} hold. Choose
\[
        \eta_t=\frac{\eta_0}{t+t_0},
\]
where $\eta_0>1/(1-\mu)$ and $t_0$ is large enough that $\eta_t\le1$ for all $t\ge0$. Then there is a constant $C_{\rm in}$, independent of $m$, $n$, and $Y_m\in\Yset$, such that
\[
        \E\norm{\bar X_m-x^*(Y_m)}^2\le \frac{C_{\rm in}}{n+t_0},
        \qquad
        \E\norm{\bar X_m-x^*(Y_m)}^4\le \frac{C_{\rm in}}{(n+t_0)^2}.
\]
\end{proposition}

The proof is standard for contractive stochastic approximation but is included in Appendix~\ref{app:inner-proof} for completeness.

\subsection{Uncorrected and corrected outer oracles}

Let $\mathcal F_m^{\rm out}$ denote the sigma-field generated by all randomness before the inner loop at outer step $m$; in particular, $Y_m$ is $\mathcal F_m^{\rm out}$-measurable. After the inner loop has been run, define
\[
        \mathcal G_m\defeq
        \mathcal F_m^{\rm out}\vee \sigma(\xi_{m,1},\ldots,\xi_{m,n}).
\]
Then $\bar X_m$ and $Y_m$ are $\mathcal G_m$-measurable. The outer slow-oracle samples below are fresh conditionally on $\mathcal G_m$.

After computing $\bar X_m$, the uncorrected slow oracle is
\begin{equation}\label{eq:oracle-raw}
        \widehat H^{\rm raw}_m
        \defeq \mathsf G(\bar X_m,Y_m;\zeta_{m+1}).
\end{equation}
The corrected slow oracle is
\begin{equation}\label{eq:oracle-corr}
        \widehat H^{\rm corr}_m
        \defeq
        \mathsf G(\bar X_m,Y_m;\zeta^g_{m+1})
        +P_*(Y_m)\bigl(\mathsf F(\bar X_m,Y_m;\zeta^f_{m+1})-\bar X_m\bigr),
\end{equation}
where fresh samples are used for the two stochastic terms. Conditional on the post-inner, pre-outer-sample sigma-field $\mathcal G_m$,
\[
        \E[\widehat H^{\rm raw}_m\mid\mathcal G_m]
        =H^{\rm raw}(\bar X_m,Y_m),
\]
and
\[
        \E[\widehat H^{\rm corr}_m\mid\mathcal G_m]
        =H^{\rm corr}(\bar X_m,Y_m).
\]
The bias vectors $H^{\rm raw}(\bar X_m,Y_m)-h(Y_m)$ and $H^{\rm corr}(\bar X_m,Y_m)-h(Y_m)$ are $\mathcal G_m$-measurable. By \Cref{prop:bias-order,prop:inner}, their second moments satisfy
\begin{equation}\label{eq:bias-sizes}
\begin{aligned}
        \E\norm{H^{\rm raw}(\bar X_m,Y_m)-h(Y_m)}^2
        &\le \frac{C}{n},\\
        \E\norm{H^{\rm corr}(\bar X_m,Y_m)-h(Y_m)}^2
        &\le \frac{C}{n^2}.
\end{aligned}
\end{equation}
The stochastic noise in both oracles has uniformly bounded second moment, by \Cref{ass:moments} and boundedness of $P_*$.

\subsection{Tikhonov-regularized outer recursion}

Fix an anchor $u\in\Yset$. Given a slow-oracle estimate $\widehat H_m$, define the outer update
\begin{equation}\label{eq:tikhonov-update}
        Y_{m+1}
        =\Pi_{\Yset}\left[
        Y_m+\beta\bigl(\widehat H_m-Y_m+\lambda(u-Y_m)\bigr)
        \right].
\end{equation}
The parameters $\beta$ and $\lambda$ are fixed during the horizon $N$ and chosen below as powers of $N$.

The regularized mean operator is
\[
        T_\lambda(y)\defeq \frac{h(y)+\lambda u}{1+
        \lambda}.
\]
It is a contraction with constant $(1+
\lambda)^{-1}$. Indeed, for any $y,z\in\Yset$, non-expansiveness of $h$ gives
\[
        \norm{T_\lambda(y)-T_\lambda(z)}
        =\frac{1}{1+\lambda}\norm{h(y)-h(z)}
        \le \frac{1}{1+\lambda}\norm{y-z}.
\]
Its unique fixed point $y_\lambda\in\Yset$ satisfies
\begin{equation}\label{eq:ylambda}
        h(y_\lambda)-y_\lambda+\lambda(u-y_\lambda)=0.
\end{equation}
Thus $y_\lambda$ is an $O(\lambda)$-residual point for the original map:
\begin{equation}\label{eq:regularization-bias}
        \norm{h(y_\lambda)-y_\lambda}
        =\lambda\norm{y_\lambda-u}
        \le \lambda D_y.
\end{equation}
The Tikhonov term is therefore an algorithmic way to make the non-expansive outer problem contractive while preserving a small residual for the original map.

\section{Main total-sample guarantees}\label{sec:main-rates}

The nested algorithms below use Tikhonov regularization to convert the non-expansive slow problem into a weakly contractive anchored problem. We first state the outer stability result used throughout the sample-complexity analysis. It separates four contributions: the Tikhonov residual $\lambda^2$, finite-time contraction, stochastic variance $\beta/\lambda$, and squared oracle bias $\epsilon_H^2/\lambda^2$. We then instantiate it with the uncorrected and corrected nested oracles.

\begin{theorem}[One-sample Tikhonov-KM with biased stochastic oracle]\label{thm:tikhonov-abstract}
Let $h:\Yset\to\Yset$ be non-expansive, let $u\in\Yset$, and let $(Y_m)$ be generated by \eqref{eq:tikhonov-update}. Let $\mathcal G_m$ be the sigma-field just before the outer oracle sample at step $m$, and suppose $Y_m$ and the bias vector $b_m$ are $\mathcal G_m$-measurable. Assume $0<\lambda\le1$ and, conditionally on $\mathcal G_m$,
\[
        \E[\widehat H_m\mid\mathcal G_m]=h(Y_m)+b_m,
        \qquad
        \E\!
        \left[
        \norm{\widehat H_m-\E[\widehat H_m\mid\mathcal G_m]}^2
        \mid\mathcal G_m
        \right]
        \le \sigma_H^2,
\]
and
\[
        \E\norm{b_m}^2\le \epsilon_H^2
        \qquad\text{for all }m=0,\ldots,N-1.
\]
There are constants $c,C>0$, depending only on $D_y$ and $\sigma_H$, such that if $\beta\le c\lambda$, then
\begin{equation}\label{eq:tikhonov-abstract-bound}
        \E\norm{h(Y_N)-Y_N}^2
        \le
        C\left(
        \lambda^2
        +\exp(-cN\beta\lambda)
        +\frac{\beta}{\lambda}
        +\frac{\epsilon_H^2}{\lambda^2}
        \right).
\end{equation}
\end{theorem}

The proof is given in Appendix~\ref{app:tikhonov-proof}. We use \Cref{thm:tikhonov-abstract} below as the outer stability result for the nested and single-loop algorithms.

The rest of this section is therefore a bias-calculation exercise. The uncorrected oracle has squared bias $O(n^{-1})$, while the corrected oracle has squared bias $O(n^{-2})$. The different total-sample exponents come entirely from this change.

\subsection{Uncorrected nested TTSA: the one-quarter exponent}

Theorem~\ref{thm:uncorrected} establishes the baseline cost of the nested Tikhonov method when the slow oracle is queried at an approximate fast equilibrium without correction. The inner fast solver gives mean-square fast error $O(n^{-1})$, and because the slow map is only Lipschitz in the fast coordinate, this produces a slow-oracle bias with squared size $O(n^{-1})$. The resulting total-sample rate remains $T^{-1/4+o(1)}$.

\begin{theorem}[Nested uncorrected Tikhonov-KM]\label{thm:uncorrected}
Suppose \Cref{ass:fast-contract,ass:slow-nonexp,ass:compact,ass:moments,ass:g-lip} hold. Fix $b\in(0,3/4)$ and set
\[
        \beta_N=N^{-b},
        \qquad
        \lambda_N=N^{-b/3},
        \qquad
        n_N=\left\lceil N^{4b/3}\right\rceil.
\]
For all sufficiently large horizons $N$, run the nested algorithm \eqref{eq:inner-loop}, \eqref{eq:oracle-raw}, \eqref{eq:tikhonov-update} for $N$ outer iterations with $n_N$ inner fast samples per outer step. Then
\begin{equation}\label{eq:uncorr-N-rate}
        \E\norm{h(Y_N)-Y_N}^2
        \le
        C_b N^{-2b/3}.
\end{equation}
The constant $C_b$ may depend on the fixed problem and oracle constants, including $\mu,\sigma,D_x,D_y,L_g$, and the inner-loop stepsize constants, but not on $N$. The total number of primitive oracle calls satisfies
\[
        T_N\asymp N^{1+4b/3}.
\]
The restriction $b<3/4$ ensures that the effective Tikhonov contraction horizon diverges:
\[
        N\beta_N\lambda_N=N^{1-4b/3}\to\infty.
\]
The choice $n_N=\lceil N^{4b/3}\rceil$ then balances the uncorrected bias contribution $n_N^{-1}/\lambda_N^2$ with the Tikhonov and variance terms.
Consequently, for every $\eps>0$, by choosing $b$ sufficiently close to $3/4$ from below,
\begin{equation}\label{eq:uncorr-T-rate}
        \E\norm{h(Y_N)-Y_N}^2
        \le
        C_\eps T_N^{-1/4+\eps}.
\end{equation}
\end{theorem}

The proof is given in Appendix~\ref{app:nested-rate-proofs}.

This baseline theorem shows that Tikhonov regularization alone does not remove the one-quarter total-sample scaling. The remaining bottleneck is the first-order slow-oracle bias induced by the approximate fast solve.

\subsection{Bias-corrected nested TTSA: the one-third exponent}

The corrected oracle changes exactly this term. \Cref{prop:bias-order} shows that the preconditioned residual cancels the first-order dependence on $\bar X_m-x^*(Y_m)$. Hence the same inner solver produces squared slow-oracle bias $O(n^{-2})$. Theorem~\ref{thm:corrected} shows that this deterministic bias improvement survives the stochastic Tikhonov outer recursion and yields the total-sample exponent $1/3$.

\begin{theorem}[Nested bias-corrected Tikhonov-KM]\label{thm:corrected}
Suppose \Cref{ass:fast-contract,ass:slow-nonexp,ass:compact,ass:moments,ass:smooth,ass:precond} hold. Fix $b\in(0,3/4)$ and set
\[
        \beta_N=N^{-b},
        \qquad
        \lambda_N=N^{-b/3},
        \qquad
        n_N=\left\lceil N^{2b/3}\right\rceil.
\]
For all sufficiently large horizons $N$, run the nested algorithm \eqref{eq:inner-loop}, \eqref{eq:oracle-corr}, \eqref{eq:tikhonov-update} for $N$ outer iterations with $n_N$ inner fast samples per outer step. Then
\begin{equation}\label{eq:corr-N-rate}
        \E\norm{h(Y_N)-Y_N}^2
        \le
        C_b N^{-2b/3}.
\end{equation}
The constant $C_b$ may depend on the fixed problem and oracle constants, including $\mu,\sigma,D_x,D_y,L_1,L_2$, the uniform bound on $P_*$, and the inner-loop stepsize constants, but not on $N$. The total number of primitive $\mathsf F$- and $\mathsf G$-oracle calls satisfies
\[
        T_N\asymp N^{1+2b/3}.
\]
This count treats the structured preconditioner query $y\mapsto P_*(y)$ as a separate oracle and does not include the cost of computing or estimating $P_*(y)$. If $P_*(y)$ is implemented by another stochastic procedure, that cost must be added separately; the theorem applies unchanged only when the direct product conditions in \Cref{ass:precond} hold at the stated scale. Again, $b<3/4$ ensures $N\beta_N\lambda_N=N^{1-4b/3}\to\infty$; here the corrected squared bias is $O(n_N^{-2})$, so $n_N=\lceil N^{2b/3}\rceil$ suffices.
Consequently, for every $\eps>0$, by choosing $b$ sufficiently close to $3/4$ from below,
\begin{equation}\label{eq:corr-T-rate}
        \E\norm{h(Y_N)-Y_N}^2
        \le
        C_\eps T_N^{-1/3+\eps}.
\end{equation}
\end{theorem}

The proof is given in Appendix~\ref{app:nested-rate-proofs}.

Thus the $T^{-1/3+o(1)}$ theorem is not merely an intermediate rate. It provides a total-oracle-complexity certificate that the proposed correction removes first-order fast-tracking leakage. The remaining limitation is the nested architecture itself: a new inner fast solve is performed for each outer slow point. \Cref{sec:single-loop-half} asks whether that cost can be reduced by tracking both the fast equilibrium and the leakage preconditioner online.

\begin{remark}[Why the comparison is fair]
Both \Cref{thm:uncorrected,thm:corrected} count all inner fast samples and all outer slow/residual samples. The improvement from $T^{-1/4+o(1)}$ to $T^{-1/3+o(1)}$ is therefore a total-oracle comparison. It is caused exactly by the change from first-order slow bias, $\epsilon_H^2\asymp n^{-1}$, to second-order slow bias, $\epsilon_H^2\asymp n^{-2}$.
\end{remark}

\section{A single-loop route to the one-half exponent with a learned preconditioner}\label{sec:single-loop-half}

The nested corrected theorem leaves open whether the inner-loop cost must be paid independently at every outer point. This section shows that, in a smooth derivative-oracle model, it need not be. A single-loop algorithm can track the moving fast equilibrium and the moving leakage preconditioner while simultaneously running the Tikhonov slow update. This changes the cost accounting: instead of spending $n_N$ fast samples per outer step, the method spends only $O(1)$ primitive samples per iteration.

Algorithm \eqref{eq:online-precond-algo} is a single-loop method. It keeps one running fast iterate $X_k$, one running slow iterate $Y_k$, and one running preconditioner iterate $P_k$; each iteration updates all three objects once. The improvement to $T^{-1/2+o(1)}$ is not obtained by variance reduction. The stochastic slow oracle still has constant conditional variance, and the Tikhonov analysis still pays the usual $\beta/\lambda$ variance term. The gain comes from avoiding repeated inner solves: the fast state and the leakage preconditioner are tracked online as $Y_k$ moves.

This is still a structured-oracle model. The theorem requires differentiability in the fast coordinate, compact stabilization, moving-target regularity of $x^*$ and $P_*$, bounded moments, and unbiased derivative samples. Under these conditions, the proof in Appendix~\ref{app:single-loop-proofs} shows that the fast equilibrium and leakage preconditioner can be tracked online at the accuracy required by the Tikhonov outer recursion.

For a horizon $N$, choose constants $\alpha,\gamma,\beta,\lambda>0$ and run
\begin{equation}\label{eq:online-precond-algo}
\begin{aligned}
        X_{k+1}
        &=\Pi_{\Xset}\left[
        X_k+\alpha\bigl(\mathsf F(X_k,Y_k;\xi^x_{k+1})-X_k\bigr)
        \right],\\
        P_{k+1}
        &=\Pi_{\Pset}\left[
        P_k+\gamma\bigl(\mathsf C(X_k,Y_k;\omega^C_{k+1})
        -P_k\mathsf A(X_k,Y_k;\omega^A_{k+1})\bigr)
        \right],\\
        \widehat H^{\rm on}_{k+1}
        &=\mathsf G(X_k,Y_k;\zeta^g_{k+1})
        +P_k\bigl(\mathsf F(X_k,Y_k;\zeta^f_{k+1})-X_k\bigr),\\
        Y_{k+1}
        &=\Pi_{\Yset}\left[
        Y_k+\beta\bigl(\widehat H^{\rm on}_{k+1}-Y_k+\lambda(u-Y_k)\bigr)
        \right].
\end{aligned}
\end{equation}
All samples in the four lines are conditionally independent given the past. Each iteration uses a constant number of primitive samples: one fast update sample, one slow sample, one residual sample, and one pair of derivative samples.

Define
\begin{equation}\label{eq:lambda-theta-scales}
        \Lambda\defeq \alpha+\left(\frac{\beta}{\alpha}\right)^2,
        \qquad
        \Theta\defeq \gamma+\Lambda+\left(\frac{\beta}{\gamma}\right)^2 .
\end{equation}
For the finite-horizon theorem we also use the burn-in time
\begin{equation}\label{eq:burnin}
        K_{\rm burn}(N)
        \defeq
        \left\lceil
        \frac{C_{\rm burn}}{\alpha\wedge\gamma}\log N
        \right\rceil,
\end{equation}
where $C_{\rm burn}$ is a sufficiently large numerical constant depending only on the fixed problem constants. The burn-in is needed because the preconditioner tracker is driven by derivative observations at $X_k$, and early fast-tracking transients feed into the preconditioner recursion.

The technical tracking estimates are stated and proved in Appendix~\ref{app:single-loop-proofs}. Their combined content is as follows: after the burn-in time, the fast tracking error has second and fourth moments of order $\Lambda$ and $\Lambda^2$, the preconditioner error has second and fourth moments of order $\Theta$ and $\Theta^2$, and the corrected online oracle has squared bias $O(\Lambda^2+\Theta\Lambda)$. These estimates show that the single-loop algorithm can maintain the same second-order corrected slow-oracle bias without restarting an inner solver. The next theorem balances this online tracking bias with the Tikhonov variance and anchoring terms.

\begin{theorem}[Single-loop Tikhonov TTSA with learned preconditioner]\label{thm:single-loop-half}
Suppose \Cref{ass:fast-contract,ass:slow-nonexp,ass:compact,ass:moments,ass:smooth,ass:online-precond,ass:moving-target} hold. For every sufficiently small $\eps>0$, choose
\[
        \alpha=N^{-1/2+\eps},
        \qquad
        \gamma=N^{-1/2+\eps},
        \qquad
        \beta=N^{-3/4+3\eps},
        \qquad
        \lambda=N^{-1/4+\eps},
\]
and run \eqref{eq:online-precond-algo} for $N$ iterations, with $N$ large enough that the parameters are in $(0,1]$ and $\beta\le c\lambda$. Then
\[
        \E\norm{h(Y_N)-Y_N}^2
        \le
        C_\eps N^{-1/2+6\eps}.
\]
Since each iteration uses $O(1)$ primitive oracle calls, equivalently
\[
        \E\norm{h(Y_N)-Y_N}^2
        \le
        C_\eps T^{-1/2+6\eps},
\]
where $T$ is the total number of primitive oracle calls.
\end{theorem}

The proof is given in Appendix~\ref{app:single-loop-proofs}.

\begin{remark}[Comparison with the nested theorem]
The $1/2$ exponent should be interpreted as an online-tracking improvement over the nested architecture; the key change is that the algorithm no longer spends $n_N$ fresh fast samples per slow update. Unlike the nested corrected theorem, this route does not assume direct access to $P_*(Y_k)$; it learns the leakage preconditioner online from derivative-oracle samples.
\end{remark}

\section{Conclusion}
This paper isolates the source of the one-quarter residual exponent in raw non-expansive two-time-scale stochastic approximation. The fixed-schedule KM lower bound shows that the underlying non-expansive slow dynamics already have a sharp residual scale for a prescribed slow schedule, while the raw fast recursion introduces first-order fast-manifold leakage into the slow oracle. If the usual separation condition $\beta_k^2/\alpha_k^3\lesssim1$ is violated, the moving-target lag itself enters as a first-order inexact-KM perturbation and remains rate-limiting.

The residual-preconditioned oracle changes this leakage from first order to second order. In a nested Tikhonov-KM scheme, this gives a total primitive-sample improvement from $T^{-1/4+o(1)}$ to $T^{-1/3+o(1)}$. In the richer derivative-oracle model, online tracking of both the fast equilibrium and the leakage preconditioner avoids repeated inner solves and yields the $T^{-1/2+o(1)}$ exponent.

\bibliographystyle{plain}
\bibliography{refs}
\newpage

\appendix

\section{Proofs of Theorem~\ref{thm:KM-lower} and Corollary~\ref{cor:KM-poly}}\label{app:km-proof}

This appendix contains the elementary rotation construction behind \Cref{thm:KM-lower} and its polynomial-schedule consequence.

\begin{proof}[Proof of \Cref{thm:KM-lower}]
Take a one-dimensional, decoupled fast recursion with $f(x,y)=0$ and compact fast set $\Xset=\{0\}$. Then $f(\cdot,y)$ is a contraction with constant $0$, and $x^*(y)=0$ for all $y$. This construction isolates the slow dynamics by entirely eliminating any first-order fast-tracking leakage.

Let the slow set $\Yset$ be the closed unit disk in $\R^2$. For a planar rotation $R_\theta$ by angle $\theta$, define the full slow mean field by
\[
        g(x,y)=R_\theta y,
        \qquad x\in\Xset,\ y\in\Yset .
\]
Then the reduced slow map strictly applies the rotation:
\[
        h(y)=g(x^*(y),y)=R_\theta y.
\]
The disk is invariant under $R_\theta$. If we assume $\theta\notin2\pi\mathbb Z$, the map $h$ is non-expansive on $\Yset$ and has a unique fixed point at the origin, giving $Y^*=\{0\}$.

To track the trajectory multiplicatively, identify $\R^2$ with the complex plane $\mathbb C$ and start at the boundary of the slow set by taking $\norm{Y_0}=1$. The exact Krasnoselskii-Mann (KM) recursion $Y_{k+1} = (1-\beta_k)Y_k + \beta_k h(Y_k)$ takes the form:
\[
        Y_{k+1} = (1-\beta_k)Y_k + \beta_k e^{i\theta} Y_k = \bigl(1-\beta_k+\beta_k e^{i\theta}\bigr)Y_k.
\]
Unrolling this recursion from $k=0$ to $N-1$ yields the final iterate:
\[
        Y_N=\prod_{k=0}^{N-1}\bigl(1-\beta_k+\beta_k e^{i\theta}\bigr)Y_0.
\]
Define $\delta$ as the squared chordal distance of the rotation:
\[
        \delta\defeq |1-e^{i\theta}|^2 = (1-\cos\theta)^2 + \sin^2\theta = 2 - 2\cos\theta = 4\sin^2(\theta/2).
\]
We can now isolate the squared magnitude of the complex multiplier at each step. By expanding the squared modulus $|z|^2 = z \bar{z}$, we have:
\begin{align*}
        |1-\beta_k+\beta_k e^{i\theta}|^2 
        &= (1-\beta_k + \beta_k\cos\theta)^2 + (\beta_k\sin\theta)^2 \\
        &= (1-\beta_k)^2 + 2\beta_k(1-\beta_k)\cos\theta + \beta_k^2(\cos^2\theta + \sin^2\theta) \\
        &= 1 - 2\beta_k + \beta_k^2 + 2\beta_k(1-\beta_k)\cos\theta + \beta_k^2 \\
        &= 1 - 2\beta_k(1-\beta_k) + 2\beta_k(1-\beta_k)\cos\theta \\
        &= 1 - \beta_k(1-\beta_k)\bigl[2(1-\cos\theta)\bigr].
\end{align*}
Substituting $\delta = 2(1-\cos\theta)$, this becomes:
\[
        |1-\beta_k+\beta_k e^{i\theta}|^2
        =1-\delta\beta_k(1-\beta_k).
\]
Because $\norm{Y_0}^2 = 1$, the squared norm of the $N$-th iterate is exactly the product of these squared multipliers:
\[
        \norm{Y_N}^2
        =\prod_{k=0}^{N-1}\bigl(1-\delta\beta_k(1-\beta_k)\bigr).
\]
Next, we evaluate the mean-square residual. The operator $(R_\theta-I)$ corresponds to multiplication by $(e^{i\theta}-1)$ in the complex plane. Its squared operator norm is exactly $|e^{i\theta}-1|^2 = \delta$. Thus,
\[
        \norm{h(Y_N)-Y_N}^2
        =\norm{(R_\theta-I)Y_N}^2
        =|e^{i\theta}-1|^2 \norm{Y_N}^2
        =\delta\norm{Y_N}^2.
\]
We now establish the minimax finite-horizon bound. We are free to select an adversarial rotation angle $\theta$ to actively penalize the chosen horizon $N$. We tune $\theta$ such that:
\[
        \delta=\frac{1}{2B_N}.
\]
This selection is valid because the condition $B_N\ge1/8$ implies $\delta\le4$, keeping it within the geometric bounds of $4\sin^2(\theta/2) \in [0,4]$. Substituting this choice of $\delta$ fixes the sum of the sequence:
\[
        \sum_{k=0}^{N-1}\delta\beta_k(1-\beta_k)=\delta B_N = \frac12.
\]
Since $\beta_k \in (0,1)$ and $\delta > 0$, all summands $a_k \defeq \delta\beta_k(1-\beta_k)$ are strictly nonnegative. Because their sum is $1/2 \le 1$, we have $a_k \in [0,1]$. This allows us to apply the elementary Weierstrass product inequality $\prod_k(1-a_k)\ge1-\sum_k a_k$ to bound the final norm from below:
\[
        \prod_{k=0}^{N-1}\bigl(1-\delta\beta_k(1-\beta_k)\bigr)\ge 1 - \sum_{k=0}^{N-1}\delta\beta_k(1-\beta_k) = 1 - \frac12 = \frac12.
\]
Consequently, substituting this lower bound back into the residual equation provides the final strict limit:
\[
        \norm{h(Y_N)-Y_N}^2
        \ge \delta \left(\frac12\right)
        =\left(\frac{1}{2B_N}\right)\frac12
        =\frac{1}{4B_N}.
\]
\end{proof}

\begin{proof}[Proof of \Cref{cor:KM-poly}]
By $\beta_k\asymp (k+1)^{-b}$ there are constants $0<c_1\le c_2<\infty$ and an index $k_0$ such that
\[
        c_1(k+1)^{-b}\le \beta_k\le c_2(k+1)^{-b},
        \qquad k\ge k_0 .
\]
Since $b>0$, $\beta_k\to0$. After increasing $k_0$ if necessary, we may assume $\beta_k\le 1/2$ for all $k\ge k_0$. Hence
\[
        \frac12\beta_k\le \beta_k(1-\beta_k)\le \beta_k,
        \qquad k\ge k_0 .
\]
Consequently
\[
        B_N=\sum_{k=0}^{N-1}\beta_k(1-\beta_k)
        \asymp \sum_{k=1}^{N} k^{-b}.
\]
For $0<b<1$, the integral comparison
\[
        \int_1^N x^{-b}\,dx
        \le \sum_{k=1}^N k^{-b}
        \le 1+\int_1^N x^{-b}\,dx
\]
shows that $\sum_{k=1}^N k^{-b}\asymp N^{1-b}$. Thus $B_N\asymp N^{1-b}$ and, in particular, $B_N\ge1/8$ for all sufficiently large $N$. Applying \Cref{thm:KM-lower},
\[
        \norm{h(Y_N)-Y_N}^2
        \ge \frac{1}{4B_N}
        \ge cN^{-(1-b)}
\]
for a constant $c>0$ independent of $N$.
\end{proof}

\section{Proofs of Propositions~\ref{prop:raw-separation-source} and~\ref{prop:raw-lag-dominates-no-win}}\label{app:raw-separation-proof}

This appendix proves \Cref{prop:raw-separation-source,prop:raw-lag-dominates-no-win}. The argument is included to make explicit what the standard raw separation condition does and does not mean. The condition $\beta_k^2/\alpha_k^3\lesssim1$ is the balance that keeps moving-target lag at the intrinsic fast statistical scale. If this balance is not imposed, the lag does not disappear; in the raw Lipschitz-only recursion it appears as a first-order slow perturbation and remains rate-limiting.

\begin{proof}[Proof of \Cref{prop:raw-separation-source}]
We first derive the generic local tracking scale. Let
\[
        e_k=X_k-x^*(Y_k),\qquad x^*_k=x^*(Y_k),
\]
and let $\bar X_{k+1}$ denote the fast update that uses $Y_k$ before comparing to the next target. Since $x^*_k\in\Xset$ and projection onto a closed convex set is non-expansive,
\[
        \norm{\bar X_{k+1}-x^*_k}
        \le
        \norm{X_k+\alpha_k(\mathsf F(X_k,Y_k;\xi_{k+1})-X_k)-x^*_k}.
\]
Writing
\[
        \Delta_{k+1}=\mathsf F(X_k,Y_k;\xi_{k+1})-f(X_k,Y_k),
\]
we have, conditionally on the past,
\[
        \bar X_{k+1}-x^*_k
        = e_k+\alpha_k(f(X_k,Y_k)-X_k)+\alpha_k\Delta_{k+1}.
\]
The deterministic part satisfies
\[
\begin{aligned}
        e_k+\alpha_k(f(X_k,Y_k)-X_k)
        &= (1-\alpha_k)e_k
           +\alpha_k(f(X_k,Y_k)-f(x^*_k,Y_k)),
\end{aligned}
\]
and therefore, by the $\mu$-contraction of $f(\cdot,Y_k)$,
\[
        \norm{e_k+\alpha_k(f(X_k,Y_k)-X_k)}
        \le (1-(1-\mu)\alpha_k)\norm{e_k}.
\]
Using conditional unbiasedness and the conditional second-moment bound of the fast noise, for all sufficiently small $\alpha_k$,
\begin{equation}\label{eq:appendix-frozen-target}
        \E\bigl[\norm{\bar X_{k+1}-x^*_k}^2\mid\F_k\bigr]
        \le (1-c_0\alpha_k)\norm{e_k}^2+C\alpha_k^2 .
\end{equation}
The actual tracking error after the slow update is
\[
        e_{k+1}=X_{k+1}-x^*(Y_{k+1})
        = (X_{k+1}-x^*_k)+(x^*_k-x^*(Y_{k+1})).
\]
Under the standard moving-target regularity used in raw finite-time TTSA estimates, $x^*$ is Lipschitz and the slow increment satisfies
\begin{equation}\label{eq:slow-increment-ms}
        \E\bigl[\norm{Y_{k+1}-Y_k}^2\mid \F_k\bigr]\le C\beta_k^2.
\end{equation}
For example, \eqref{eq:slow-increment-ms} follows from bounded or uniformly square-integrable slow updates. Hence
\[
        \E\bigl[\norm{x^*(Y_{k+1})-x^*_k}^2\mid\F_k\bigr]\le C\beta_k^2 .
\]
Apply the inequality
\[
        \norm{a+d}^2\le (1+\eta)\norm a^2+(1+\eta^{-1})\norm d^2
\]
with $\eta=c_0\alpha_k/4$. Combining it with \eqref{eq:appendix-frozen-target} gives, after decreasing $c_0$ and increasing $C$,
\begin{equation}\label{eq:appendix-moving-target-recursion}
        u_{k+1}\le (1-c\alpha_k)u_k+C\alpha_k^2+C\frac{\beta_k^2}{\alpha_k},
        \qquad
        u_k\defeq \E\norm{e_k}^2 .
\end{equation}
For constant stepsizes, the steady-state consequence of \eqref{eq:appendix-moving-target-recursion} is immediate:
\[
        u_k\lesssim \alpha+\frac{\beta^2}{\alpha^2}
\]
after a transient of order $1/\alpha$. For regular polynomial stepsizes, the same local scale follows by the usual comparison argument on windows of length proportional to $1/\alpha_k$: over such a window the stepsizes change by only constant factors, while the multiplicative factor $\prod(1-c\alpha_i)$ contracts by a fixed amount. This proves the upper tracking scale used in \Cref{subsec:raw-separation-origin}.

We next show that both pieces of this scale are sharp. The stochastic floor $\alpha$ is already present when the slow variable is fixed. Take the scalar recursion
\[
        X_{k+1}=(1-\rho\alpha)X_k+\alpha W_{k+1},
        \qquad 0<\rho\le1,
\]
with $\E W_{k+1}=0$ and $\E W_{k+1}^2=\sigma^2>0$. This is the fast recursion for the linear contraction $f(x,0)=(1-\rho)x$ with fixed target $x^*(0)=0$. In stationarity,
\[
        \E X_k^2
        = \frac{\alpha^2\sigma^2}{1-(1-\rho\alpha)^2}
        = \frac{\alpha\sigma^2}{2\rho-\rho^2\alpha}
        = \Theta(\alpha).
\]
Thus the fast statistical term cannot be improved in general.

It remains to show that the moving-target term $(\beta/\alpha)^2$ is also real and can occur without any noise. Work in the complex plane, identified with $\R^2$. Fix a rotation angle $\theta\in(0,\pi]$ and define
\[
        h(y)=e^{i\theta}y,
        \qquad
        f(x,y)=(1-\rho)x+\rho y,
        \qquad 0<\rho\le1.
\]
Then $f(\cdot,y)$ is a contraction with constant $1-\rho$, and $x^*(y)=y$. Consider the exact KM slow update with a constant slow stepsize $\beta$:
\[
        Y_{k+1}=mY_k,
        \qquad
        m\defeq 1+\beta(e^{i\theta}-1),
        \qquad
        Y_0=1.
\]
The noiseless fast update with constant stepsize $\alpha$ is
\[
        X_{k+1}=(1-\rho\alpha)X_k+\rho\alpha Y_k.
\]
Let $a=1-\rho\alpha$ and $e_k=X_k-Y_k$. Starting from $e_0=0$, we have
\[
\begin{aligned}
        e_{k+1}
        &= X_{k+1}-Y_{k+1}                                      \\
        &= (1-\rho\alpha)X_k+\rho\alpha Y_k-mY_k                 \\
        &= a e_k+(1-m)Y_k .
\end{aligned}
\]
Since $Y_k=m^k$, solving this linear recursion gives
\begin{equation}\label{eq:exact-moving-lag}
        e_k=(1-m)\sum_{j=0}^{k-1}a^{k-1-j}m^j
        =(1-m)\frac{m^k-a^k}{m-a} .
\end{equation}
Assume $0<\beta\le c_\theta\alpha$, where $c_\theta>0$ is sufficiently small depending only on $\theta$ and $\rho$, and choose $k_\alpha=\lceil C_0/\alpha\rceil$ with $C_0$ sufficiently large. Then $a^{k_\alpha}\le 1/4$. Because $|m|\le1$ and $m^{k}-1=(m-1)\sum_{j=0}^{k-1}m^j$, we have $|m^k-1|\le k|m-1|$. Also $\beta k_\alpha\le 2c_\theta C_0$ when $\alpha$ is small. Choosing $c_\theta$ smaller if necessary therefore ensures $|m^{k_\alpha}-1|\le1/4$. Hence
\[
        |m^{k_\alpha}-a^{k_\alpha}|\ge \frac12.
\]
Moreover,
\[
        |1-m|=\beta|1-e^{i\theta}|,
        \qquad
        |m-a|=|\rho\alpha+\beta(e^{i\theta}-1)|\le C\alpha .
\]
Substituting these estimates into \eqref{eq:exact-moving-lag} yields
\[
        |e_{k_\alpha}|
        \ge c\frac{\beta}{\alpha},
        \qquad
        |e_{k_\alpha}|^2
        \ge c\left(\frac{\beta}{\alpha}\right)^2 .
\]
Thus the deterministic lag behind a non-expansive KM-generated moving target has precisely the order $\beta/\alpha$ in norm. This proves that the moving-target contribution in \eqref{eq:appendix-moving-target-recursion} cannot be removed from a worst-case raw tracking estimate.

Finally, the transfer from fast tracking error to raw slow bias is also sharp under a Lipschitz-only assumption. For the same reduced map $h$ and equilibrium $x^*(y)=y$, define locally, or globally on the compact stabilized region,
\[
        g_L(x,y)=h(y)+L(x-y)
\]
with a fixed nonzero linear map $L$. Then $g_L(x^*(y),y)=h(y)$, but
\[
        g_L(X_k,Y_k)-h(Y_k)=L(X_k-x^*(Y_k)).
\]
Thus the raw slow-oracle bias can be exactly first order in the fast tracking error. Combining this observation with the two lower-bound examples above gives the worst-case raw squared perturbation scale
\[
        \alpha+\left(\frac{\beta}{\alpha}\right)^2.
\]
To keep the moving-target part from exceeding the intrinsic fast statistical scale $\alpha$, one must require
\[
        \left(\frac{\beta}{\alpha}\right)^2\lesssim\alpha,
\]
which is equivalent to $\beta^2/\alpha^3\lesssim1$. This proves \Cref{prop:raw-separation-source}.
\end{proof}

\begin{proof}[Proof of \Cref{prop:raw-lag-dominates-no-win}]
The point of the proposition is only to record the exponent arithmetic once the moving-target lag is allowed to dominate. The raw tracking scale in \Cref{prop:raw-separation-source} gives a root-mean-square lag of order
\[
        \eta_k\asymp \frac{\beta_k}{\alpha_k}\asymp (k+1)^{-(b-a)} .
\]
Because the raw slow oracle is only Lipschitz in the fast coordinate, the slow update can see an inexact-KM perturbation of this same first order. We recall the elementary inexact-KM algebra to make the weighting explicit. Consider a deterministic perturbed KM step
\[
        Y_{k+1}=Y_k+\beta_k\bigl(p(Y_k)+d_k\bigr),
        \qquad p(y)=h(y)-y,
\]
with $\norm{d_k}\lesssim\eta_k$. Let $y^*\in\Fix(h)$ and let
\[
        Z_{k+1}=Y_k+\beta_k p(Y_k)=(1-\beta_k)Y_k+\beta_k h(Y_k)
\]
be the exact KM point. The standard KM identity for a non-expansive $h$ gives
\[
        \norm{Z_{k+1}-y^*}^2
        \le
        \norm{Y_k-y^*}^2
        -\beta_k(1-\beta_k)\norm{p(Y_k)}^2 .
\]
Since the iterates are stabilized on a compact set, expansion of
$Y_{k+1}=Z_{k+1}+\beta_k d_k$ gives
\[
        \norm{Y_{k+1}-y^*}^2
        \le
        \norm{Z_{k+1}-y^*}^2
        +C\beta_k\norm{d_k}+\beta_k^2\norm{d_k}^2 .
\]
Summing from $k=0$ to $N-1$ and using that $\beta_k\eta_k\to0$ when $a<b$ shows that the perturbation contributes, up to constants and lower-order quadratic terms, the weighted scale
\[
        E_N^{\rm lag}
        =
        \frac{\sum_{k<N}\beta_k\eta_k}{\sum_{k<N}\beta_k(1-\beta_k)} .
\]
This calculation is the deterministic part of the usual inexact-KM residual estimate; here we use it only as a scale diagnostic for the raw lag-dominated regime. For $b\in(0,1)$, the denominator satisfies
\[
        \sum_{k<N}\beta_k(1-\beta_k)\asymp \sum_{k<N}(k+1)^{-b}\asymp N^{1-b}.
\]
The numerator is
\[
        \sum_{k<N}\beta_k\eta_k
        \asymp
        \sum_{k<N}(k+1)^{-b}(k+1)^{-(b-a)}
        =
        \sum_{k<N}(k+1)^{-(2b-a)} .
\]
There are three cases. If $2b-a<1$, the numerator is of order $N^{1-2b+a}$, and therefore
\[
        E_N^{\rm lag}\asymp \frac{N^{1-2b+a}}{N^{1-b}}=N^{-(b-a)}.
\]
If $2b-a=1$, the numerator is of order $\log N$, while $b-a=1-b$, so
\[
        E_N^{\rm lag}\asymp (\log N)N^{-(1-b)}=(\log N)N^{-(b-a)}.
\]
If $2b-a>1$, the numerator is bounded uniformly in $N$, and hence
\[
        E_N^{\rm lag}\asymp N^{-(1-b)}.
\]
Combining the three cases gives, up to the logarithm on the boundary,
\[
        E_N^{\rm lag}\asymp N^{-\min\{1-b,\,b-a\}} .
\]
The exact KM lower-bound scale from \Cref{cor:KM-poly} is $N^{-(1-b)}$. Therefore the raw exponent permitted by the exact KM term and by the lag-induced first-order perturbation is at most
\[
        \min\{1-b,\,b-a\}.
\]
For a fixed $a>1/2$, the function $b\mapsto\min\{1-b,b-a\}$ is maximized when the two arguments are equal, namely at
\[
        1-b=b-a,
        \qquad\text{so}\qquad
        b=\frac{1+a}{2}.
\]
The resulting value is
\[
        1-b=b-a=\frac{1-a}{2}.
\]
Since $a>1/2$, this value is strictly smaller than $1/4$, and it approaches $1/4$ only as $a\downarrow1/2$. Taking the supremum over admissible $a$ gives $1/4$. Thus allowing the moving-target lag to dominate does not improve the raw one-quarter scaling; it only moves the bottleneck from the fast statistical floor to the first-order lag perturbation.
\end{proof}

\section{Proofs of Lemma~\ref{lem:newton-point} and Proposition~\ref{prop:bias-order}}\label{app:bias-proofs}

This appendix verifies the two deterministic second-order identities used in \Cref{sec:correction}.

\begin{proof}[Proof of \Cref{lem:newton-point}]
Fix $y$ and write
\[
        x^*=x^*(y),\qquad e=x-x^*,\qquad r(x,y)=f(x,y)-x,
        \qquad A_x=I-\nabla_x f(x,y).
\]
The equilibrium condition is $r(x^*,y)=0$. Since $\norm{\nabla_x f(x,y)}\le\mu<1$, the matrix $A_x$ is invertible. Indeed,
\[
        \norm{A_xv}
        =\norm{v-\nabla_x f(x,y)v}
        \ge (1-\mu)\norm v,
\]
so $\norm{A_x^{-1}}\le(1-\mu)^{-1}$.

Apply Taylor's theorem to $r(\cdot,y)$ around $x$. Since $\nabla_x r(x,y)=\nabla_x f(x,y)-I=-A_x$, there is a remainder $\rho$ such that
\[
        0=r(x^*,y)
        =r(x,y)+\nabla_x r(x,y)(x^*-x)+\rho
        =r(x,y)+A_xe+\rho .
\]
The derivative of $r$ in the fast coordinate has the same Lipschitz constant as $\nabla_x f$, so the integral remainder formula gives
\[
        \norm{\rho}
        \le \int_0^1 L_f(1-t)\norm{x^*-x}^2\,dt
        =\frac{L_f}{2}\norm e^2 .
\]
Rearranging the Taylor identity,
\[
        e+A_x^{-1}r(x,y)=-A_x^{-1}\rho .
\]
By the definition of the implicit Newton point,
\[
        \calN(x,y)-x^*
        =x+A_x^{-1}r(x,y)-x^*
        =e+A_x^{-1}r(x,y).
\]
Therefore
\[
        \norm{\calN(x,y)-x^*(y)}
        \le \norm{A_x^{-1}}\norm{\rho}
        \le \frac{L_f}{2(1-\mu)}\norm{x-x^*(y)}^2 .
\]
Finally, if $g(\cdot,y)$ is $L_N$-Lipschitz on a region containing the segment from $x^*(y)$ to $\calN(x,y)$, then
\[
\begin{aligned}
        \norm{g(\calN(x,y),y)-h(y)}
        &=\norm{g(\calN(x,y),y)-g(x^*(y),y)}\\
        &\le L_N\norm{\calN(x,y)-x^*(y)}\\
        &\le \frac{L_NL_f}{2(1-\mu)}\norm{x-x^*(y)}^2 .
\end{aligned}
\]
\end{proof}

\begin{proof}[Proof of \Cref{prop:bias-order}]
Fix $x\in\Xset$ and $y\in\Yset$, and write
\[
        x^*=x^*(y),\qquad e=x-x^*.
\]
For the raw oracle,
\[
        H^{\rm raw}(x,y)-h(y)=g(x,y)-g(x^*(y),y),
\]
so \Cref{ass:g-lip} gives
\[
        \norm{H^{\rm raw}(x,y)-h(y)}
        \le L_g\norm e .
\]

For the corrected oracle, expand both $g$ and the fast residual around the fast equilibrium. Taylor's theorem in the fast variable gives
\[
        g(x^*+e,y)=g(x^*,y)+C(y)e+R_g(e,y)
        =h(y)+C(y)e+R_g(e,y),
\]
where
\[
        R_g(e,y)
        =\int_0^1\bigl(\nabla_x g(x^*+te,y)-\nabla_xg(x^*,y)\bigr)e\,dt .
\]
The Lipschitz bound on $\nabla_xg$ implies
\[
        \norm{R_g(e,y)}
        \le \int_0^1 L_2t\norm e^2\,dt
        \le \frac{L_2}{2}\norm e^2 .
\]
Similarly,
\[
        f(x^*+e,y)=f(x^*,y)+\nabla_x f(x^*,y)e+R_f(e,y),
        \qquad
        \norm{R_f(e,y)}\le \frac{L_2}{2}\norm e^2 .
\]
Because $f(x^*,y)=x^*$, subtracting $x^*+e$ gives
\[
        f(x^*+e,y)-(x^*+e)
        =(\nabla_xf(x^*,y)-I)e+R_f(e,y)
        =-A(y)e+R_f(e,y).
\]
Substituting the two Taylor expansions into \eqref{eq:Hcorr},
\[
\begin{aligned}
        H^{\rm corr}(x,y)-h(y)
        &=C(y)e+R_g(e,y)+P_*(y)\bigl(-A(y)e+R_f(e,y)\bigr)\\
        &=\bigl(C(y)-P_*(y)A(y)\bigr)e+R_g(e,y)+P_*(y)R_f(e,y).
\end{aligned}
\]
Since $P_*(y)=C(y)A(y)^{-1}$, the first-order term vanishes:
\[
        C(y)-P_*(y)A(y)=C(y)-C(y)A(y)^{-1}A(y)=0.
\]
Thus
\[
        H^{\rm corr}(x,y)-h(y)=R_g(e,y)+P_*(y)R_f(e,y).
\]
Moreover,
\[
        \norm{P_*(y)}
        \le \norm{C(y)}\norm{A(y)^{-1}}
        \le \frac{L_1}{1-\mu}.
\]
Consequently
\[
        \norm{H^{\rm corr}(x,y)-h(y)}
        \le \frac{L_2}{2}\left(1+\frac{L_1}{1-\mu}\right)\norm e^2.
\]
Taking $L_c$ to dominate both this constant and $L_g$ proves the proposition.
\end{proof}

\section{Proof of Proposition~\ref{prop:inner}}\label{app:inner-proof}

We prove \Cref{prop:inner}. The proof is for a fixed slow point $y\in\Yset$. All constants are uniform in $y$, because the fast contraction, the compact sets, and the oracle moment bounds are uniform.

\begin{proof}[Proof of \Cref{prop:inner}]
Fix $y\in\Yset$ and write $x^*=x^*(y)$. Let
\[
        e_t=X_t-x^*,
        \qquad
        \Delta_{t+1}=\mathsf F(X_t,y;\xi_{t+1})-f(X_t,y).
\]
Conditional on the past $\F_t$,
\[
        \E[\Delta_{t+1}\mid\F_t]=0,
        \qquad
        \E[\norm{\Delta_{t+1}}^q\mid\F_t]\le C,\quad q=2,4 .
\]
Projection can only reduce the distance to $x^*$ because $x^*\in\Xset$ and Euclidean projection onto a closed convex set is non-expansive. Thus the same estimates are obtained by first analyzing the unprojected update and then applying non-expansiveness of the projection.

Let
\[
        a_t=e_t+\eta_t\bigl(f(X_t,y)-X_t\bigr).
\]
Since $f(x^*,y)=x^*$,
\[
        a_t=(1-\eta_t)e_t+\eta_t\bigl(f(X_t,y)-f(x^*,y)\bigr).
\]
The contraction of $f(\cdot,y)$ gives, for $\eta_t\in[0,1]$,
\[
        \norm{a_t}
        \le (1-\eta_t)\norm{e_t}+\eta_t\mu\norm{e_t}
        =\bigl(1-(1-\mu)\eta_t\bigr)\norm{e_t}.
\]
Using the conditional mean-zero property of $\Delta_{t+1}$,
\[
\begin{aligned}
        \E[\norm{e_{t+1}}^2\mid\F_t]
        &\le
        \E[\norm{a_t+\eta_t\Delta_{t+1}}^2\mid\F_t]\\
        &=\norm{a_t}^2
          +2\eta_t\ip{a_t}{\E[\Delta_{t+1}\mid\F_t]}
          +\eta_t^2\E[\norm{\Delta_{t+1}}^2\mid\F_t]\\
        &\le
        \bigl(1-c_0\eta_t\bigr)\norm{e_t}^2+C\eta_t^2,
\end{aligned}
\]
where $c_0>0$ after increasing $t_0$ so that $\eta_t$ is small enough. Taking expectations,
\begin{equation}\label{eq:inner-second-rec}
        u_{t+1}\le (1-c_0\eta_t)u_t+C\eta_t^2,
        \qquad
        u_t\defeq\E\norm{e_t}^2.
\end{equation}

We now solve this scalar recursion explicitly. Put $s_t=t+t_0$ and $\eta_t=\eta_0/s_t$. Choose $c_0$ slightly below $2(1-\mu)$ and then choose $t_0$ large enough that the preceding one-step inequality is valid. Since $\eta_0>1/(1-\mu)$, we may arrange
\[
        a\defeq c_0\eta_0>1 .
\]
We claim $u_t\le K/s_t$ for a sufficiently large $K$. If the claim holds at time $t$, then
\[
        u_{t+1}
        \le \left(1-\frac{a}{s_t}\right)\frac{K}{s_t}+\frac{C\eta_0^2}{s_t^2}
        =\frac{K}{s_t}-\frac{aK-C\eta_0^2}{s_t^2}.
\]
Also
\[
        \frac{K}{s_t+1}
        =\frac{K}{s_t}-\frac{K}{s_t(s_t+1)}
        \ge \frac{K}{s_t}-\frac{K}{s_t^2}.
\]
Therefore $u_{t+1}\le K/s_{t+1}$ whenever $(a-1)K\ge C\eta_0^2$. Increasing $K$ to cover the finite initial value $u_0$ closes the induction and yields
\[
        \E\norm{X_t-x^*(y)}^2\le \frac{C}{t+t_0}.
\]

For the fourth moment, we use the following elementary perturbation estimate. If $a$ is $\F_t$-measurable, $\E[\Delta\mid\F_t]=0$, and $\E[\norm\Delta^4\mid\F_t]<\infty$, then
\begin{equation}\label{eq:fourth-perturb}
        \E[\norm{a+\eta\Delta}^4\mid\F_t]
        \le
        \norm a^4+C\eta^2\norm a^2\E[\norm\Delta^2\mid\F_t]+C\eta^4\E[\norm\Delta^4\mid\F_t].
\end{equation}
To verify \eqref{eq:fourth-perturb}, write
\[
        \norm{a+\eta\Delta}^2
        =\norm a^2+2\eta\ip{a}{\Delta}+\eta^2\norm\Delta^2
\]
and square. The term linear in $\ip{a}{\Delta}$ has conditional expectation zero. The terms of order $\eta^2$ are bounded by Cauchy--Schwarz:
\[
        \E[\ip{a}{\Delta}^2\mid\F_t]\le \norm a^2\E[\norm\Delta^2\mid\F_t].
\]
The cubic term is controlled by Holder and Young:
\[
\begin{aligned}
        \eta^3\E[|\ip{a}{\Delta}|\norm\Delta^2\mid\F_t]
        &\le \eta^3\norm a\,\E[\norm\Delta^3\mid\F_t]\\
        &\le C\eta^3\norm a\bigl(\E[\norm\Delta^4\mid\F_t]\bigr)^{3/4}\\
        &\le C\eta^2\norm a^2\E[\norm\Delta^2\mid\F_t]
             +C\eta^4\E[\norm\Delta^4\mid\F_t],
\end{aligned}
\]
where the last step uses the uniform moment bound to absorb constants. The quartic term is already of the desired form.

Applying \eqref{eq:fourth-perturb} with $a=a_t$ and $\eta=\eta_t$, and using the contraction of $a_t$, gives
\[
        \E[\norm{e_{t+1}}^4\mid\F_t]
        \le
        (1-c_1\eta_t)\norm{e_t}^4+C\eta_t^2\norm{e_t}^2+C\eta_t^4
\]
for a constant $c_1>0$ after increasing $t_0$. Taking expectations and using the second-moment estimate,
\[
        v_{t+1}
        \le
        (1-c_1\eta_t)v_t+\frac{C}{(t+t_0)^3},
        \qquad
        v_t\defeq\E\norm{e_t}^4.
\]
Choose $c_1$ slightly below $4(1-\mu)$ and $t_0$ sufficiently large. Since $\eta_0>1/(1-\mu)$, we can ensure
\[
        a_1\defeq c_1\eta_0>2 .
\]
We prove by induction that $v_t\le K/s_t^2$. If this holds at time $t$, then
\[
        v_{t+1}
        \le \left(1-\frac{a_1}{s_t}\right)\frac{K}{s_t^2}+\frac{C}{s_t^3}
        =\frac{K}{s_t^2}-\frac{a_1K-C}{s_t^3}.
\]
Since
\[
        \frac{K}{(s_t+1)^2}
        \ge \frac{K}{s_t^2}-\frac{2K}{s_t^3},
\]
the induction closes whenever $(a_1-2)K\ge C$. Increasing $K$ to cover $v_0$ proves
\[
        \E\norm{X_t-x^*(y)}^4\le \frac{C}{(t+t_0)^2}.
\]
Taking $t=n$ proves \Cref{prop:inner}.
\end{proof}

\section{Proof of Theorem~\ref{thm:tikhonov-abstract}}\label{app:tikhonov-proof}

\begin{proof}[Proof of \Cref{thm:tikhonov-abstract}]
Let $y_\lambda$ be the unique point satisfying \eqref{eq:ylambda}, and set
\[
        q_\lambda(y)=h(y)-y+\lambda(u-y).
\]
Then $q_\lambda(y_\lambda)=0$. We first prove the basic dissipativity inequality. For $e=y-y_\lambda$,
\[
\begin{aligned}
        \ip{e}{h(y)-h(y_\lambda)-e}
        &=\ip{e}{h(y)-h(y_\lambda)}-\norm e^2\\
        &\le \norm e\,\norm{h(y)-h(y_\lambda)}-\norm e^2\\
        &\le0,
\end{aligned}
\]
because $h$ is non-expansive. Since
\[
        q_\lambda(y)-q_\lambda(y_\lambda)
        =h(y)-h(y_\lambda)-(y-y_\lambda)-\lambda(y-y_\lambda),
\]
we obtain
\begin{equation}\label{eq:dissipative}
        \ip{y-y_\lambda}{q_\lambda(y)}
        \le -\lambda\norm{y-y_\lambda}^2 .
\end{equation}
The same identity also gives, for $\lambda\le1$,
\begin{equation}\label{eq:q-lip}
        \norm{q_\lambda(y)}
        =\norm{q_\lambda(y)-q_\lambda(y_\lambda)}
        \le (2+\lambda)\norm{y-y_\lambda}
        \le 3\norm{y-y_\lambda}.
\end{equation}

Write the oracle in the form
\[
        \widehat H_m=h(Y_m)+b_m+\xi_{m+1},
        \qquad
        \E[\xi_{m+1}\mid\mathcal G_m]=0,
        \qquad
        \E[\norm{\xi_{m+1}}^2\mid\mathcal G_m]\le\sigma_H^2.
\]
Let $e_m=Y_m-y_\lambda$. Because $y_\lambda\in\Yset$ and projection is non-expansive,
\[
\begin{aligned}
        \E[\norm{Y_{m+1}-y_\lambda}^2\mid\mathcal G_m]
        &\le
        \E\norm{e_m+\beta(q_\lambda(Y_m)+b_m+\xi_{m+1})}^2\\
        &=\norm{e_m}^2+2\beta\ip{e_m}{q_\lambda(Y_m)}
          +2\beta\ip{e_m}{b_m}\\
        &\quad
          +\beta^2\E[\norm{q_\lambda(Y_m)+b_m+\xi_{m+1}}^2\mid\mathcal G_m].
\end{aligned}
\]
The cross term with $\xi_{m+1}$ vanishes by conditional mean zero. By \eqref{eq:q-lip} and $(a+b+c)^2\le3(a^2+b^2+c^2)$,
\[
        \E[\norm{q_\lambda(Y_m)+b_m+\xi_{m+1}}^2\mid\mathcal G_m]
        \le C\norm{e_m}^2+C\norm{b_m}^2+C\sigma_H^2.
\]
Combining this with \eqref{eq:dissipative},
\[
\begin{aligned}
        \E[\norm{Y_{m+1}-y_\lambda}^2\mid\mathcal G_m]
        &\le (1-2\beta\lambda+C\beta^2)\norm{e_m}^2
           +2\beta\norm{e_m}\norm{b_m}\\
        &\quad +C\beta^2\norm{b_m}^2+C\beta^2\sigma_H^2 .
\end{aligned}
\]
Young's inequality gives
\[
        2\beta\norm{e_m}\norm{b_m}
        \le \beta\lambda\norm{e_m}^2+\frac{\beta}{\lambda}\norm{b_m}^2.
\]
If $\beta\le c\lambda$ with $c$ sufficiently small, the $C\beta^2\norm{e_m}^2$ term is absorbed into the contraction term. Thus
\[
        \E[\norm{Y_{m+1}-y_\lambda}^2\mid\mathcal G_m]
        \le (1-c\beta\lambda)\norm{e_m}^2
           +C\frac{\beta}{\lambda}\norm{b_m}^2+C\beta^2\sigma_H^2 .
\]
Taking expectations and using $\E\norm{b_m}^2\le\epsilon_H^2$,
\[
        D_{m+1}
        \le (1-c\beta\lambda)D_m
           +C\frac{\beta}{\lambda}\epsilon_H^2+C\beta^2\sigma_H^2,
        \qquad
        D_m\defeq\E\norm{Y_m-y_\lambda}^2.
\]
Iterating the scalar recursion,
\[
\begin{aligned}
        D_N
        &\le (1-c\beta\lambda)^ND_0
           +\left(C\frac{\beta}{\lambda}\epsilon_H^2+C\beta^2\sigma_H^2\right)
             \sum_{j=0}^{N-1}(1-c\beta\lambda)^j\\
        &\le e^{-cN\beta\lambda}D_0
           +C\frac{\epsilon_H^2}{\lambda^2}
           +C\frac{\beta}{\lambda}\sigma_H^2 .
\end{aligned}
\]
The initial distance is bounded by compactness, $D_0\le D_y^2$, and $\sigma_H$ is part of the fixed constants.

It remains to convert distance to $y_\lambda$ into residual for the original non-expansive map. The residual map $p(y)=h(y)-y$ is $2$-Lipschitz:
\[
        \norm{p(y)-p(z)}
        \le \norm{h(y)-h(z)}+\norm{y-z}
        \le2\norm{y-z}.
\]
Also \eqref{eq:regularization-bias} gives $\norm{p(y_\lambda)}\le\lambda D_y$. Hence
\[
        \norm{p(Y_N)}
        \le \norm{p(y_\lambda)}+\norm{p(Y_N)-p(y_\lambda)}
        \le \lambda D_y+2\norm{Y_N-y_\lambda}.
\]
After squaring and taking expectations,
\[
        \E\norm{h(Y_N)-Y_N}^2
        \le C\lambda^2+CD_N,
\]
which yields \eqref{eq:tikhonov-abstract-bound}.
\end{proof}

\section{Additional corrected-oracle moment bounds}\label{app:oracle-details}

The corrected stochastic oracle \eqref{eq:oracle-corr} has uniformly bounded second moment. Indeed,
\[
\begin{aligned}
        \widehat H^{\rm corr}_m-H^{\rm corr}(\bar X_m,Y_m)
        &=\mathsf G(\bar X_m,Y_m;\zeta^g_{m+1})-g(\bar X_m,Y_m)\\
        &\quad +P_*(Y_m)\bigl(\mathsf F(\bar X_m,Y_m;\zeta^f_{m+1})-f(\bar X_m,Y_m)\bigr).
\end{aligned}
\]
The two terms are conditionally mean-zero. By \Cref{ass:moments} and boundedness of $P_*$,
\[
        \E\bigl[\norm{\widehat H^{\rm corr}_m-H^{\rm corr}(\bar X_m,Y_m)}^2\mid\mathcal G_m\bigr]
        \le C.
\]
The same argument gives the fourth-moment bound if needed. The mean-bias calculation is exactly \Cref{prop:bias-order}; combining it with \Cref{prop:inner} gives \eqref{eq:bias-sizes}.

\section{Proofs of Theorems~\ref{thm:uncorrected} and~\ref{thm:corrected}}\label{app:nested-rate-proofs}

We instantiate the abstract Tikhonov estimate with the two nested oracles. The only difference between the two arguments is the squared bias scale: $O(n^{-1})$ for the raw query and $O(n^{-2})$ for the corrected query.

\begin{proof}[Proof of \Cref{thm:uncorrected}]
For the uncorrected oracle, \eqref{eq:bias-sizes} gives
\[
        \epsilon_H^2\le \frac{C}{n_N}.
\]
The stochastic part has uniformly bounded conditional variance by \Cref{ass:moments}. Apply \Cref{thm:tikhonov-abstract} with
\[
        \beta_N=N^{-b},\qquad
        \lambda_N=N^{-b/3},\qquad
        n_N=\lceil N^{4b/3}\rceil .
\]
For all sufficiently large $N$, the condition $\beta_N\le c\lambda_N$ holds because
\[
        \frac{\beta_N}{\lambda_N}=N^{-2b/3}\to0.
\]
We now evaluate the four terms in \eqref{eq:tikhonov-abstract-bound}. First,
\[
        \lambda_N^2=N^{-2b/3}.
\]
Second,
\[
        \frac{\beta_N}{\lambda_N}=N^{-b+b/3}=N^{-2b/3}.
\]
Third,
\[
        N\beta_N\lambda_N=N^{1-b-b/3}=N^{1-4b/3}.
\]
Because $b<3/4$, this quantity diverges to infinity, so the exponential term is smaller than any fixed inverse polynomial and can be absorbed into $CN^{-2b/3}$. Finally,
\[
        \frac{\epsilon_H^2}{\lambda_N^2}
        \le \frac{C}{n_N\lambda_N^2}
        \le C N^{-4b/3}N^{2b/3}
        =C N^{-2b/3}.
\]
Substitution gives
\[
        \E\norm{h(Y_N)-Y_N}^2\le C_bN^{-2b/3},
\]
which is \eqref{eq:uncorr-N-rate}.

The primitive sample count is
\[
        T_N=N(n_N+O(1))\asymp N^{1+4b/3}.
\]
Thus
\[
        N^{-2b/3}=T_N^{-\frac{2b/3}{1+4b/3}+o(1)}
        =T_N^{-\frac{2b}{3+4b}+o(1)}.
\]
The exponent $\frac{2b}{3+4b}$ tends to $1/4$ as $b\uparrow3/4$. Hence, given any $\epsilon>0$, choose $b<3/4$ sufficiently close to $3/4$ so that
\[
        \frac{2b}{3+4b}\ge \frac14-\epsilon .
\]
After adjusting constants, this proves \eqref{eq:uncorr-T-rate}.
\end{proof}

\begin{proof}[Proof of \Cref{thm:corrected}]
For the corrected oracle, the stochastic variance is uniformly bounded by the argument in Appendix~\ref{app:oracle-details}. The mean bias is second order in the fast error, and \eqref{eq:bias-sizes} gives
\[
        \epsilon_H^2\le \frac{C}{n_N^2}.
\]
Apply \Cref{thm:tikhonov-abstract} with
\[
        \beta_N=N^{-b},\qquad
        \lambda_N=N^{-b/3},\qquad
        n_N=\lceil N^{2b/3}\rceil .
\]
Again $\beta_N/\lambda_N=N^{-2b/3}\to0$, so the smallness condition on $\beta_N$ holds for large $N$. The Tikhonov residual and stochastic variance terms are
\[
        \lambda_N^2=N^{-2b/3},
        \qquad
        \frac{\beta_N}{\lambda_N}=N^{-2b/3}.
\]
The finite-time contraction term is negligible because
\[
        N\beta_N\lambda_N=N^{1-4b/3}\to\infty
        \qquad (b<3/4).
\]
The corrected bias contribution is
\[
        \frac{\epsilon_H^2}{\lambda_N^2}
        \le \frac{C}{n_N^2\lambda_N^2}
        \le C N^{-4b/3}N^{2b/3}
        =C N^{-2b/3}.
\]
Therefore
\[
        \E\norm{h(Y_N)-Y_N}^2\le C_bN^{-2b/3},
\]
which proves \eqref{eq:corr-N-rate}.

The primitive $\mathsf F$- and $\mathsf G$-oracle count is
\[
        T_N=N(n_N+O(1))\asymp N^{1+2b/3}.
\]
Consequently
\[
        N^{-2b/3}=T_N^{-\frac{2b/3}{1+2b/3}+o(1)}
        =T_N^{-\frac{2b}{3+2b}+o(1)}.
\]
As $b\uparrow3/4$, the exponent tends to
\[
        \frac{2(3/4)}{3+2(3/4)}=\frac13.
\]
Thus for every $\epsilon>0$ one may choose $b$ sufficiently close to $3/4$ from below to obtain \eqref{eq:corr-T-rate}. The exact preconditioner call is not counted in this primitive sample count, as stated in the theorem.
\end{proof}

\section{Proof of Theorem~\ref{thm:single-loop-half} and auxiliary online-tracking lemmas}\label{app:single-loop-proofs}

This appendix proves the tracking estimates and the one-half theorem for the single-loop learned-preconditioner algorithm. The estimates are grouped here to keep the main text focused on the algorithmic transition from nested inner solves to online tracking.

\begin{lemma}[Fast tracking with a moving slow variable]\label{lem:single-fast}
Suppose \Cref{ass:fast-contract,ass:compact,ass:moments,ass:moving-target} hold. Let $(X_k,P_k,Y_k)$ follow \eqref{eq:online-precond-algo} with $0<\alpha\le\alpha_0$, $0<\beta\le\alpha$, $0<\gamma\le1$, and $0<\lambda\le1$. Then there are constants $c,C>0$, depending only on the fixed problem constants, such that for all $k\le N$,
\[
        \E\norm{X_k-x^*(Y_k)}^2
        \le
        C e^{-c\alpha k}\E\norm{X_0-x^*(Y_0)}^2
        +C\Lambda,
\]
and
\[
        \E\norm{X_k-x^*(Y_k)}^4
        \le
        C e^{-c\alpha k}\E\norm{X_0-x^*(Y_0)}^4
        +C\Lambda^2.
\]
In particular, since $\Xset$ and $\Yset$ are compact, if $k\ge K_{\rm burn}(N)$ then
\[
        \E\norm{X_k-x^*(Y_k)}^2\le C\Lambda+CN^{-12},
        \qquad
        \E\norm{X_k-x^*(Y_k)}^4\le C\Lambda^2+CN^{-12}.
\]
\end{lemma}

\begin{proof}
Write
\[
        x_k^*=x^*(Y_k),\qquad e_k=X_k-x_k^* .
\]
Let $\widetilde X_{k+1}$ denote the fast update after the first line of \eqref{eq:online-precond-algo}, but measured against the frozen target $x_k^*$. Since $x_k^*\in\Xset$ and projection onto $\Xset$ is non-expansive,
\[
        \norm{\widetilde X_{k+1}-x_k^*}
        \le \norm{e_k+\alpha(f(X_k,Y_k)-X_k+\Delta^x_{k+1})},
\]
where
\[
        \Delta^x_{k+1}=\mathsf F(X_k,Y_k;\xi^x_{k+1})-f(X_k,Y_k),
        \qquad
        \E[\Delta^x_{k+1}\mid\F_k]=0.
\]
The deterministic frozen-target part satisfies
\[
\begin{aligned}
        \norm{e_k+\alpha(f(X_k,Y_k)-X_k)}
        &=\norm{(1-\alpha)e_k+\alpha(f(X_k,Y_k)-f(x_k^*,Y_k))}\\
        &\le (1-(1-\mu)\alpha)\norm{e_k}.
\end{aligned}
\]
Therefore, by the same second-moment expansion used in the inner-loop proof,
\begin{equation}\label{eq:single-fast-frozen-second}
        \E[\norm{\widetilde X_{k+1}-x_k^*}^2\mid\F_k]
        \le (1-c\alpha)\norm{e_k}^2+C\alpha^2
\end{equation}
for sufficiently small $\alpha$.

The slow target moves after the slow update. By the pathwise boundedness in \Cref{ass:moving-target},
\[
        \norm{Y_{k+1}-Y_k}\le B\beta,
\]
and the Lipschitz property of $x^*$ gives
\[
        d_k\defeq\norm{x^*(Y_{k+1})-x^*(Y_k)}\le L_*B\beta\le C\beta.
\]
Since $e_{k+1}=X_{k+1}-x^*(Y_{k+1})$ equals the frozen error plus this target displacement, the inequality
\[
        \norm{a+b}^2\le (1+s)\norm a^2+(1+s^{-1})\norm b^2
\]
with $s$ proportional to $\alpha$ yields
\[
\begin{aligned}
        \E[\norm{e_{k+1}}^2\mid\F_k]
        &\le (1-c\alpha/2)\norm{e_k}^2+C\alpha^2+C\frac{\beta^2}{\alpha}.
\end{aligned}
\]
Taking expectations and iterating the scalar recursion gives
\[
        \E\norm{e_k}^2
        \le C e^{-c\alpha k}\E\norm{e_0}^2
          +C\frac{\alpha^2+\beta^2/\alpha}{\alpha}
        = C e^{-c\alpha k}\E\norm{e_0}^2
          +C\left(\alpha+\left(\frac{\beta}{\alpha}\right)^2\right).
\]
This is the second-moment estimate because the expression in parentheses is $\Lambda$.

We next prove the fourth-moment estimate. The frozen-target fourth-moment expansion gives
\[
        \E[\norm{\widetilde X_{k+1}-x_k^*}^4\mid\F_k]
        \le (1-c\alpha)\norm{e_k}^4+C\alpha^2\norm{e_k}^2+C\alpha^4.
\]
The middle term is not simply discarded. It is absorbed into the fourth-moment drift by Young's inequality: for all $t\ge0$,
\[
        C\alpha^2t\le \frac{c\alpha}{4}t^2+C\alpha^3.
\]
Taking $t=\norm{e_k}^2$ gives
\begin{equation}\label{eq:single-fast-frozen-fourth}
        \E[\norm{\widetilde X_{k+1}-x_k^*}^4\mid\F_k]
        \le (1-c\alpha)\norm{e_k}^4+C\alpha^3.
\end{equation}
Now account for the moving target. The fourth-power version of Young's inequality,
\[
        \norm{a+b}^4\le (1+s)\norm a^4+Cs^{-3}\norm b^4,
\]
with $s$ proportional to $\alpha$, together with $d_k\le C\beta$, gives
\[
        \E[\norm{e_{k+1}}^4\mid\F_k]
        \le (1-c\alpha/2)\norm{e_k}^4+C\alpha^3+C\frac{\beta^4}{\alpha^3}.
\]
Iterating,
\[
\begin{aligned}
        \E\norm{e_k}^4
        &\le C e^{-c\alpha k}\E\norm{e_0}^4
          +C\frac{\alpha^3+\beta^4/\alpha^3}{\alpha}\\
        &= C e^{-c\alpha k}\E\norm{e_0}^4
          +C\left(\alpha^2+\left(\frac{\beta}{\alpha}\right)^4\right)
        \le C e^{-c\alpha k}\E\norm{e_0}^4+C\Lambda^2.
\end{aligned}
\]
The final inequality follows from $\Lambda=\alpha+(\beta/\alpha)^2$.

Finally, compactness bounds $\norm{e_0}^2$ and $\norm{e_0}^4$ deterministically. If $k\ge K_{\rm burn}(N)$, then in particular $k\ge C_{\rm burn}\alpha^{-1}\log N$ because $\alpha\wedge\gamma\le\alpha$. Choosing $C_{\rm burn}$ large enough makes $e^{-c\alpha k}\le N^{-12}$, proving the post-burn-in bounds.
\end{proof}

\begin{lemma}[Online tracking of the leakage preconditioner]\label{lem:precond-track}
Suppose \Cref{ass:fast-contract,ass:compact,ass:moments,ass:smooth,ass:online-precond,ass:moving-target} hold. Let
\[
        E_k\defeq P_k-P_*(Y_k),
        \qquad
        e_k\defeq X_k-x^*(Y_k),
\]
and let $(X_k,P_k,Y_k)$ follow \eqref{eq:online-precond-algo}. For sufficiently small $\gamma$, there are constants $c,C>0$ such that, for all $k\le N$,
\begin{align}
        \E\norm{E_k}_F^2
        &\le
        C e^{-c\gamma k}\E\norm{E_0}_F^2
        +C\Phi_k\E\norm{e_0}^2
        +C\Theta, \label{eq:precond-second-transient}\\
        \E\norm{E_k}_F^4
        &\le
        C e^{-c\gamma k}\E\norm{E_0}_F^4
        +C\Phi_k\E\norm{e_0}^4
        +C\Theta^2, \label{eq:precond-fourth-transient}
\end{align}
where
\begin{equation}\label{eq:phi-transient}
        \Phi_k
        \defeq
        \gamma\sum_{i=0}^{k-1}e^{-c\gamma(k-1-i)}e^{-c\alpha i}.
\end{equation}
Consequently, if $k\ge K_{\rm burn}(N)$, then compactness of $\Xset,\Yset,\Pset$ implies
\[
        \E\norm{E_k}_F^2\le C\Theta+CN^{-12},
        \qquad
        \E\norm{E_k}_F^4\le C\Theta^2+CN^{-12}.
\]
\end{lemma}

\begin{proof}
Abbreviate
\[
        A_k^*=A(Y_k),\qquad C_k^*=C(Y_k),\qquad P_k^*=P_*(Y_k),
\]
and
\[
        \bar A_k=A_x(X_k,Y_k),\qquad \bar C_k=C_x(X_k,Y_k).
\]
Because $C_k^*=P_k^*A_k^*$ and $P_k=P_k^*+E_k$, the conditional mean drift decomposes as
\[
\begin{aligned}
        \bar C_k-P_k\bar A_k
        &=\bar C_k-(P_k^*+E_k)\bar A_k\\
        &=C_k^*-P_k^*A_k^*-E_kA_k^*
          +(\bar C_k-C_k^*)-P_k^*(\bar A_k-A_k^*)-E_k(\bar A_k-A_k^*)\\
        &=-E_kA_k^*+R_k,
\end{aligned}
\]
where
\[
        R_k=(\bar C_k-C_k^*)-P_k^*(\bar A_k-A_k^*)-E_k(\bar A_k-A_k^*).
\]
The derivative Lipschitz bounds imply
\[
        \norm{\bar A_k-A_k^*}_F+\norm{\bar C_k-C_k^*}_F\le C\norm{e_k}.
\]
Also $P_k,P_k^*\in\Pset$, so $E_k$ is uniformly bounded. Thus the product term $E_k(\bar A_k-A_k^*)$ is still first order in $\norm{e_k}$, and
\begin{equation}\label{eq:R-bound}
        \norm{R_k}_F\le C\norm{e_k}.
\end{equation}

The linear term $E\mapsto EA_k^*$ is coercive. For a row vector $v$,
\[
        vA_k^*v^\top
        =\norm v^2-v\nabla_x f(x^*(Y_k),Y_k)v^\top
        \ge (1-\mu)\norm v^2,
\]
because $\norm{\nabla_x f}\le\mu$. Summing over the rows of a matrix $E$ gives
\begin{equation}\label{eq:A-coercive}
        \ip{E}{EA_k^*}_F\ge (1-\mu)\norm E_F^2 .
\end{equation}
Since $A_k^*$ is uniformly bounded, \eqref{eq:A-coercive} implies that for sufficiently small $\gamma$,
\begin{equation}\label{eq:precond-linear-contract}
        \norm{E-\gamma EA_k^*}_F^2\le (1-c\gamma)\norm E_F^2,
        \qquad
        \norm{E-\gamma EA_k^*}_F^4\le (1-c\gamma)\norm E_F^4 .
\end{equation}

Let $M_{k+1}$ denote the centered derivative-oracle noise:
\[
        M_{k+1}
        =
        \bigl(\mathsf C(X_k,Y_k;\omega^C_{k+1})-\bar C_k\bigr)
        -P_k\bigl(\mathsf A(X_k,Y_k;\omega^A_{k+1})-\bar A_k\bigr).
\]
Then $\E[M_{k+1}\mid\F_k]=0$. By \Cref{ass:online-precond} and the uniform bound $P_k\in\Pset$,
\begin{equation}\label{eq:M-moments}
        \E[\norm{M_{k+1}}_F^q\mid\F_k]\le C,\qquad q=2,4 .
\end{equation}
Before projection and before the target moves, the error relative to the frozen target $P_k^*$ is
\[
        E_k+\gamma(-E_kA_k^*+R_k+M_{k+1}).
\]
Projection onto $\Pset$ is non-expansive relative to $P_k^*\in\Pset$. Expanding the square, using the conditional mean-zero property of $M_{k+1}$, and then using \eqref{eq:precond-linear-contract}, \eqref{eq:R-bound}, and Young's inequality, we get
\[
        \E[\norm{\widetilde P_{k+1}-P_k^*}_F^2\mid\F_k]
        \le (1-c\gamma)\norm{E_k}_F^2+C\gamma\norm{e_k}^2+C\gamma^2.
\]
The target moves from $P_k^*=P_*(Y_k)$ to $P_{k+1}^*=P_*(Y_{k+1})$. Since $P_*$ is Lipschitz and the slow increment is bounded by $B\beta$,
\[
        \norm{P_{k+1}^*-P_k^*}_F\le C\beta.
\]
Using the two-term Young inequality with parameter proportional to $\gamma$,
\[
        \E[\norm{E_{k+1}}_F^2\mid\F_k]
        \le (1-c\gamma/2)\norm{E_k}_F^2
             +C\gamma\norm{e_k}^2+C\gamma^2+C\frac{\beta^2}{\gamma}.
\]
Taking expectations and using \Cref{lem:single-fast},
\[
\begin{aligned}
        \E\norm{E_{k+1}}_F^2
        &\le (1-c\gamma/2)\E\norm{E_k}_F^2
          +C\gamma e^{-c\alpha k}\E\norm{e_0}^2\\
        &\quad +C\gamma\Lambda+C\gamma^2+C\frac{\beta^2}{\gamma}.
\end{aligned}
\]
Iterating this recursion yields \eqref{eq:precond-second-transient}. The stationary contribution is
\[
        C\frac{\gamma\Lambda+\gamma^2+\beta^2/\gamma}{\gamma}
        =C\left(\Lambda+\gamma+\left(\frac{\beta}{\gamma}\right)^2\right)
        =C\Theta .
\]

We now prove the fourth-moment estimate. The frozen linear contraction in \eqref{eq:precond-linear-contract}, the deterministic perturbation bound \eqref{eq:R-bound}, and
\[
        \norm{a+b}^4\le(1+s)\norm a^4+Cs^{-3}\norm b^4
\]
with $s$ proportional to $\gamma$ give
\[
        \norm{E_k-\gamma E_kA_k^*+\gamma R_k}_F^4
        \le (1-c\gamma/2)\norm{E_k}_F^4+C\gamma\norm{e_k}^4.
\]
Now add the martingale noise $\gamma M_{k+1}$. Applying the fourth-moment perturbation inequality \eqref{eq:fourth-perturb} in Frobenius norm,
\[
\begin{aligned}
        \E[\norm{E_k-\gamma E_kA_k^*+\gamma R_k+\gamma M_{k+1}}_F^4\mid\F_k]
        &\le (1-c\gamma/2)\norm{E_k}_F^4+C\gamma\norm{e_k}^4\\
        &\quad +C\gamma^2\norm{E_k-\gamma E_kA_k^*+\gamma R_k}_F^2+C\gamma^4.
\end{aligned}
\]
The middle term is absorbed by Young's inequality:
\[
        C\gamma^2\norm{E_k}_F^2
        \le \frac{c\gamma}{8}\norm{E_k}_F^4+C\gamma^3,
\]
and the pieces involving $R_k$ are bounded by $C\gamma\norm{e_k}^4+C\gamma^3$ using \eqref{eq:R-bound} and compactness. Hence, relative to the frozen target,
\[
        \E[\norm{\widetilde P_{k+1}-P_k^*}_F^4\mid\F_k]
        \le (1-c\gamma)\norm{E_k}_F^4+C\gamma\norm{e_k}^4+C\gamma^3.
\]
The target movement contributes $C\beta^4/\gamma^3$ by the fourth-power Young inequality with parameter proportional to $\gamma$. Therefore
\[
        \E[\norm{E_{k+1}}_F^4\mid\F_k]
        \le (1-c\gamma/2)\norm{E_k}_F^4
             +C\gamma\norm{e_k}^4+C\gamma^3+C\frac{\beta^4}{\gamma^3}.
\]
Taking expectations and using the fourth-moment estimate of \Cref{lem:single-fast},
\[
\begin{aligned}
        \E\norm{E_{k+1}}_F^4
        &\le (1-c\gamma/2)\E\norm{E_k}_F^4
          +C\gamma e^{-c\alpha k}\E\norm{e_0}^4\\
        &\quad +C\gamma\Lambda^2+C\gamma^3+C\frac{\beta^4}{\gamma^3}.
\end{aligned}
\]
Iteration gives \eqref{eq:precond-fourth-transient}. The stationary contribution is
\[
        C\frac{\gamma\Lambda^2+\gamma^3+\beta^4/\gamma^3}{\gamma}
        =C\left(\Lambda^2+\gamma^2+\left(\frac{\beta}{\gamma}\right)^4\right)
        \le C\Theta^2 .
\]

It remains to simplify the transient terms after burn-in. Let $m=\alpha\wedge\gamma$. We claim that, after reducing the constant $c$ if necessary,
\[
        \Phi_k\le C e^{-cmk}.
\]
To see this, split the sum defining $\Phi_k$ into $i\le k/2$ and $i>k/2$. On the first part, $k-1-i\ge k/2-1$, so the factor $e^{-c\gamma(k-1-i)}$ gives $e^{-c'\gamma k}$ and the remaining geometric sum is bounded after multiplication by $\gamma$. On the second part, $i>k/2$, so $e^{-c\alpha i}\le e^{-c'\alpha k}$, and again the remaining $\gamma$-weighted geometric sum is bounded. Since $m\le\alpha,\gamma$, both parts are bounded by $Ce^{-cmk}$.

If $k\ge K_{\rm burn}(N)$ and $C_{\rm burn}$ is large enough, then $e^{-c\gamma k}\le N^{-12}$ and $\Phi_k\le CN^{-12}$. Compactness bounds the initial moments of $E_0$ and $e_0$, so \eqref{eq:precond-second-transient} and \eqref{eq:precond-fourth-transient} imply the stated post-burn-in estimates.
\end{proof}

\begin{lemma}[Post-burn-in bias of the learned-preconditioner oracle]\label{lem:online-bias}
Under the assumptions of \Cref{lem:single-fast,lem:precond-track}, define
\[
        \bar H_k\defeq g(X_k,Y_k)+P_k(f(X_k,Y_k)-X_k).
\]
If $k\ge K_{\rm burn}(N)$, then
\[
        \E\norm{\bar H_k-h(Y_k)}^2
        \le
        C\bigl(\Lambda^2+\Theta\Lambda\bigr)+CN^{-10}.
\]
Moreover, the stochastic part of $\widehat H^{\rm on}_{k+1}$ has uniformly bounded conditional second moment.
\end{lemma}

\begin{proof}
Let
\[
        e_k=X_k-x^*(Y_k),
        \qquad
        r_k=f(X_k,Y_k)-X_k.
\]
Insert and subtract the ideal preconditioner:
\[
\begin{aligned}
        \bar H_k-h(Y_k)
        &=\bigl[g(X_k,Y_k)+P_*(Y_k)r_k-h(Y_k)\bigr]\\
        &\quad+\bigl(P_k-P_*(Y_k)\bigr)r_k.
\end{aligned}
\]
By the deterministic cancellation estimate in \Cref{prop:bias-order},
\[
        \norm{g(X_k,Y_k)+P_*(Y_k)r_k-h(Y_k)}
        \le C\norm{e_k}^2.
\]
Therefore the squared expectation of the first term is at most
\[
        C\E\norm{e_k}^4\le C\Lambda^2+CN^{-12}
\]
after burn-in, by \Cref{lem:single-fast}.

For the second term, use the fact that the residual vanishes at the fast fixed point:
\[
\begin{aligned}
        \norm{r_k}
        &=\norm{f(X_k,Y_k)-X_k-f(x^*(Y_k),Y_k)+x^*(Y_k)}\\
        &\le \norm{f(X_k,Y_k)-f(x^*(Y_k),Y_k)}+\norm{X_k-x^*(Y_k)}\\
        &\le (1+\mu)\norm{e_k}.
\end{aligned}
\]
By Cauchy--Schwarz,
\[
\begin{aligned}
        \E\norm{(P_k-P_*(Y_k))r_k}^2
        &\le
        \left(\E\norm{P_k-P_*(Y_k)}_F^4\right)^{1/2}
        \left(\E\norm{r_k}^4\right)^{1/2}\\
        &\le C(\Theta+N^{-6})(\Lambda+N^{-6})\\
        &\le C\Theta\Lambda+CN^{-10}.
\end{aligned}
\]
The last line uses the uniform boundedness supplied by compactness; the precise power $N^{-10}$ is arbitrary and only needs to be negligible compared with the target rates.

Finally, the stochastic part of the online oracle is
\[
        \bigl(\mathsf G(X_k,Y_k;\zeta^g_{k+1})-g(X_k,Y_k)\bigr)
        +P_k\bigl(\mathsf F(X_k,Y_k;\zeta^f_{k+1})-f(X_k,Y_k)\bigr).
\]
Conditional on $\F_k$, the samples are fresh and have uniformly bounded second moments by \Cref{ass:moments}; $P_k$ is uniformly bounded because it is projected onto $\Pset$. Hence the conditional second moment of the stochastic part is uniformly bounded.
\end{proof}

\begin{proof}[Proof of \Cref{thm:single-loop-half}]
Let $K_b=K_{\rm burn}(N)$. For the theorem's parameter choice $\alpha=\gamma=N^{-1/2+\eps}$,
\[
        K_b=O(N^{1/2-\eps}\log N)=o(N).
\]
For every $k\ge K_b$, \Cref{lem:online-bias} gives
\[
        \E\norm{\E[\widehat H^{\rm on}_{k+1}\mid\F_k]-h(Y_k)}^2
        \le C(\Lambda^2+\Theta\Lambda)+CN^{-10},
\]
and the same lemma gives a uniform conditional variance bound for the martingale part of the oracle.

Apply the abstract Tikhonov theorem, \Cref{thm:tikhonov-abstract}, to the shifted process beginning at time $K_b$ with initial point $Y_{K_b}$ and remaining horizon $M=N-K_b$. The proof of \Cref{thm:tikhonov-abstract} only uses the conditional bias and variance bounds along the interval on which it is applied, and compactness uniformly bounds the new initial distance. Therefore
\[
        \E\norm{h(Y_N)-Y_N}^2
        \le
        C\left(
        \lambda^2+
        \exp(-c\beta\lambda (N-K_b))+
        \frac{\beta}{\lambda}
        +\frac{\Lambda^2+\Theta\Lambda+N^{-10}}{\lambda^2}
        \right).
\]

It remains to evaluate the terms. With
\[
        \alpha=\gamma=N^{-1/2+\eps},
        \qquad
        \beta=N^{-3/4+3\eps},
        \qquad
        \lambda=N^{-1/4+\eps},
\]
we have
\[
        \frac{\beta}{\alpha}=N^{-1/4+2\eps},
        \qquad
        \frac{\beta}{\gamma}=N^{-1/4+2\eps}.
\]
Thus
\[
        \Lambda
        =\alpha+\left(\frac{\beta}{\alpha}\right)^2
        \le C N^{-1/2+4\eps},
        \qquad
        \Theta
        =\gamma+\Lambda+\left(\frac{\beta}{\gamma}\right)^2
        \le C N^{-1/2+4\eps}.
\]
Consequently,
\[
        \frac{\Lambda^2+\Theta\Lambda+N^{-10}}{\lambda^2}
        \le
        C\frac{N^{-1+8\eps}}{N^{-1/2+2\eps}}+CN^{-10}N^{1/2}
        \le C N^{-1/2+6\eps}
\]
for all sufficiently large $N$ and fixed sufficiently small $\eps$. The Tikhonov residual and stochastic variance terms are
\[
        \lambda^2=N^{-1/2+2\eps},
        \qquad
        \frac{\beta}{\lambda}=N^{-1/2+2\eps},
\]
and these are both bounded by $N^{-1/2+6\eps}$ for $N\ge1$.

Finally,
\[
        \beta\lambda(N-K_b)
        =N^{-1+4\eps}(N-K_b).
\]
Since $K_b=o(N)$, for all sufficiently large $N$ this is at least $cN^{4\eps}$. The exponential term is therefore super-polynomially small and is absorbed into the displayed polynomial bound. We obtain
\[
        \E\norm{h(Y_N)-Y_N}^2\le C_\eps N^{-1/2+6\eps}.
\]
Each iteration of \eqref{eq:online-precond-algo} uses a fixed number of primitive samples, independent of $N$. Hence $T\asymp N$, and the same estimate holds with $T$ in place of $N$ after changing the constant.
\end{proof}

\section{Horizon tuning and anytime conversion}\label{app:anytime}

The main theorems are stated for a known horizon $N$. This is only for notational clarity. To obtain an anytime algorithm, run epochs of lengths $N_j=2^j$ and initialize each epoch from the previous endpoint. For the corrected nested method, set
\[
        \beta_j=N_j^{-b},
        \qquad
        \lambda_j=N_j^{-b/3},
        \qquad
        n_j=\lceil N_j^{2b/3}\rceil .
\]
For the uncorrected nested method, use the same $\beta_j$ and $\lambda_j$ but set $n_j=\lceil N_j^{4b/3}\rceil$. For the single-loop learned-preconditioner method, set the horizon-tuned parameters in \Cref{thm:single-loop-half} with $N$ replaced by $N_j$ in each epoch. The proof of \Cref{thm:tikhonov-abstract} applies within each epoch; the contraction term exponentially suppresses the previous epoch's initialization error. Summing the geometrically increasing epoch costs changes the final total-sample rate only by absolute constants, and using a non-exact endpoint horizon changes it by at most logarithmic factors.

\end{document}